\documentclass[accepted]{uai2026} 

\usepackage[american]{babel}

\usepackage{natbib}

\usepackage{xurl}
\usepackage{mathtools}
\usepackage{amssymb,amsthm}
\usepackage{booktabs}
\usepackage{graphicx}
\usepackage{algorithm}
\usepackage{algorithmic}
\usepackage{tikz}
\usepackage[font=small]{caption}
\usetikzlibrary{arrows.meta,positioning,decorations.pathreplacing}

\definecolor{safegrn}{HTML}{27AE60}
\definecolor{badred}{HTML}{E74C3C}
\definecolor{missorg}{HTML}{F39C12}
\definecolor{selblue}{HTML}{003E74}
\definecolor{ltgrn}{rgb}{0.88,1.0,0.88}
\definecolor{ltred}{rgb}{1.0,0.88,0.88}
\definecolor{ltblu}{rgb}{0.88,0.94,1.0}
\definecolor{ltbg}{HTML}{F5F5F5}
\colorlet{MAGENTA}{magenta}
\newcommand{\okcov}[1]{\mathbf{#1}}
\newcommand{\bestw}[1]{\mathbf{#1}}
\newcommand{\secondw}[1]{\underline{#1}}

\title{Partial Causal Structure Learning for Valid Selective Conformal Inference under Interventions}

\author[1]{\href{mailto:amir.asiaeetaheri@vumc.org?Subject=Your UAI 2026 paper}{Amir~Asiaee}{}}
\author[2]{Kaveh Aryan}
\author[3]{James P.\ Long}
\affil[1]{Department of Biostatistics, Vanderbilt University Medical Center, TN 37203, USA}
\affil[2]{Department of Informatics, King's College London, London, UK}
\affil[3]{Department of Biostatistics, MD Anderson Cancer Center, Houston, TX, USA}

\date{}

\newtheorem{theorem}{Theorem}

\newtheorem{proposition}{Proposition}
\newtheorem{corollary}{Corollary}
\theoremstyle{definition}
\newtheorem{definition}{Definition}
\newtheorem{assumption}{Assumption}
\theoremstyle{remark}
\newtheorem{remark}{Remark}
\newtheorem{example}{Example}

\newcommand{\E}{\mathbb{E}}
\newcommand{\Pp}{\mathbb{P}}
\newcommand{\R}{\mathbb{R}}

\newcommand{\1}{\mathbf{1}}

\newcommand{\calA}{\mathcal{A}}
\newcommand{\calS}{\mathcal{S}}
\newcommand{\calU}{\mathcal{U}}
\newcommand{\calB}{\mathcal{B}}
\newcommand{\calL}{\mathcal{L}}
\newcommand{\calP}{\mathcal{P}}
\newcommand{\desc}{\mathrm{desc}}
\newcommand{\Cand}{\mathrm{Cand}}
\newcommand{\FPR}{\mathrm{FPR}}
\newcommand{\FNR}{\mathrm{FNR}}

\begin{document}
\maketitle


\begin{abstract}
Selective conformal prediction can yield substantially tighter uncertainty sets when we can identify calibration examples that are exchangeable with the test example.
In interventional settings, such as perturbation experiments in genomics, exchangeability often holds only within subsets of interventions that leave a target variable ``unaffected'' (e.g., non-descendants of an intervened node in a causal graph).
{We study the practical regime where this invariance structure is unknown and must be estimated from data. Our main result quantifies how coverage degrades when the estimated safe calibration set accidentally includes interventions that affect the target, and gives a conservative correction when an upper bound on this error is available. Rather than learning a full causal graph, we learn only the intervention--target relationships needed to choose calibration interventions. We give algorithms for this partial learning task and evaluate them on synthetic structural equation models and Replogle K562 CRISPR-interference data, where the experiments illustrate synthetic gains from selective calibration and finite-sample tradeoffs on real perturbation screens.}
\end{abstract}

\section{Introduction}\label{sec:intro}

Conformal prediction (CP) provides distribution-free uncertainty quantification: given any black-box predictor, CP constructs prediction sets with finite-sample marginal coverage guarantees under exchangeability \citep{vovk2005algorithmic,shafer2008tutorial,angelopoulos2023gentle}.
Split conformal prediction, in particular, requires only a single pass over a held-out calibration set and scales to modern high-dimensional problems.

However, the marginal coverage guarantee of standard CP can be loose when the data exhibit heterogeneity.
In many scientific applications, data are generated under multiple interventions or environments, and exchangeability holds \emph{only within} certain subsets of these conditions.
If we can identify such subsets, \emph{selective} (or \emph{Mondrian}) calibration \citep{vovk2012conditional,bostrom2021mondrian} provides the same coverage guarantee but with prediction sets that are often substantially tighter, because calibration is restricted to the relevant stratum.

\textbf{Motivating application: gene perturbation experiments.}
Single-cell perturbation screens such as Perturb-seq \citep{dixit2016perturbseq,replogle2022genomewide} now routinely profile the transcriptional consequences of hundreds to thousands of genetic interventions (e.g., CRISPRi knockdowns) in a single experiment.
A central scientific question is: for a given target gene $i$, which interventions affect its expression and by how much?
Under a causal model, intervening on gene $a$ changes the distribution of gene $i$ if and only if $i$ is a causal descendant of $a$ in the gene regulatory network. 


For non-descendant targets, the residual distribution under intervention $a$ is the same as under control, creating a natural exchangeability structure.
Exploiting this structure via selective conformal inference can provide tighter, more informative prediction intervals for the effects of unseen interventions{,} enabling better prioritization in experimental design and more confident in silico differential expression calls.

The challenge is that the descendant structure is rarely known.
Full causal graph learning in high dimensions is notoriously difficult and computationally expensive \citep{peters2017elements,hauser2012interventional}, and the resulting errors propagate to the selective calibration procedure in ways that are poorly understood.
{We quantify the statistical cost of imperfect causal knowledge for conformal inference and develop algorithms that learn only the partial causal structure needed to choose calibration interventions.}

\textbf{Contributions.}
{Our contributions are:}
\begin{enumerate}[leftmargin=*,itemsep=3pt]
\item {\textbf{Coverage under imperfect stratum selection.}
We prove a finite-sample coverage bound for selective conformal prediction when some selected calibration interventions come from the wrong causal stratum, and we give a conservative level correction when the amount of such error can be bounded.}

\item {\textbf{Partial causal learning for calibration.}
Instead of recovering the full causal graph, we reduce the learning problem to deciding which calibration interventions can safely be used for each target. This makes the relevant error asymmetric: discarding a safe calibration point mainly costs efficiency, while admitting an intervention that affects the target can hurt coverage.}

\item {\textbf{Algorithms and empirical validation.}
We give a descendant-discovery procedure based on intersections of differentially affected sets and a local invariant-causal-prediction procedure for estimating distance to an intervention, with recovery conditions for both. On synthetic linear SEMs, coverage degrades monotonically as the selected calibration set is contaminated, and {the corrected procedure restores conservative coverage when supplied with the realized contamination level}.}
{On Replogle K562 CRISPRi data, the corrected method exceeds nominal coverage among the causal-selector methods on its feasible cases, while a weighted CP baseline using control-cell coexpression provides a conservative full-feasibility comparison.}
\end{enumerate}

\textbf{Organization.}
Section~\ref{sec:related} discusses related work.
Section~\ref{sec:setup} formalizes the setup.
{{Section~\ref{sec:taskdriven} formulates the partial causal learning objective.}}
{Section~\ref{sec:delta} presents the $\delta$-robustness theorem.}
Section~\ref{sec:algorithms} describes the algorithms.
Section~\ref{sec:recovery} states recovery conditions.
Section~\ref{sec:experiments} details the experimental evaluation.
Section~\ref{sec:discussion} concludes.
All proofs are deferred to the Supplementary Material.

\section{Related Work}\label{sec:related}

\textbf{Conformal prediction: foundations and efficiency.}
Conformal prediction originated in the algorithmic learning theory literature \citep{vovk2005algorithmic} and has become a standard approach to distribution-free uncertainty quantification \citep{shafer2008tutorial,angelopoulos2023gentle}.
Conformalized quantile regression \citep{romano2019cqr} improves efficiency by adapting interval widths to the input, while maintaining the finite-sample coverage guarantee.

\textbf{Conditional and Mondrian conformal prediction.}
Standard CP provides only marginal coverage; conditional coverage (coverage conditional on $X$) is generally impossible without strong assumptions \citep{vovk2012conditional}.
Mondrian conformal prediction \citep{bostrom2021mondrian} provides a practical middle ground: it stratifies calibration points into categories and applies CP within each stratum, achieving stratum-conditional coverage.
\citet{gibbs2023conditional} develop a framework interpolating between marginal and conditional validity, showing how to achieve coverage guarantees over finite collections of subgroups.
Our selective calibration is Mondrian-like, with strata defined by causal invariances rather than observable features.

\textbf{Conformal prediction under distribution shift.}
When exchangeability fails, weighted conformal prediction can restore validity under covariate shift if the density ratio is known \citep{tibshirani2019covshift}.
\citet{barber2023beyond} develop conformal prediction beyond exchangeability, providing a general framework for weighted and non-exchangeable settings with explicit finite-sample coverage bounds.
The contamination model we study is complementary: rather than a continuous distributional shift, we analyze the discrete setting where a fraction of calibration points come from the wrong stratum.

\textbf{Robust conformal prediction under contamination.}
\citet{einbinder2024labelnoise} study label noise robustness, showing that conformal prediction is conservative when noise is dispersive and providing corrections for adversarial noise of bounded size.
\citet{clarkson2024contamination} study split conformal prediction under a Huber contamination model, providing coverage bounds in terms of the contamination fraction and the Kolmogorov--Smirnov distance.
\citet{gendler2022adversarial} address adversarial input perturbations via randomized smoothing.
Our Theorem~\ref{thm:delta} is tailored to the \emph{selective} (Mondrian) setting where contamination arises from misclassification of calibration strata, which is the natural error mode when causal structure is learned.
The key distinction is that we bound coverage loss as a function of the fraction of miscategorized calibration points, without assumptions on the contaminating distribution.

\textbf{Conformal prediction and causal inference.}
\citet{lei2021counterfactuals} develop conformal inference of counterfactuals and individual treatment effects in the potential outcomes framework, providing finite-sample coverage under unconfoundedness.
\citet{alaa2023metalearners} extend this to conformal meta-learners for predictive inference of individualized treatment effects.
By contrast, we use causal structure to identify approximately exchangeable calibration subsets across \emph{multiple simultaneous interventions}, enabling tighter predictive intervals for post-intervention outcomes rather than performing inference on a treatment effect.

\textbf{Causal discovery from interventional data.}
\citet{hauser2012interventional,hauser2015joint} characterize interventional Markov equivalence classes and develop greedy algorithms for learning directed acyclic graphs (DAGs) from mixed data.
\citet{eberhardt2005number} establish that $\lceil\log_2(N)\rceil + 1$ experiments suffice to identify the full causal graph over $N$ variables in the worst case, motivating efficient experimental design.
\citet{squires2020active} propose active structure learning via directed clique trees, achieving near-optimal intervention counts for full graph identification.
\citet{brouillard2020differentiable} develop differentiable causal discovery methods that scale to larger graphs using interventional data.
These works aim to recover the complete DAG; {our approach instead learns the partial structure needed for selective conformal calibration}.

\textbf{Invariant causal prediction.}
Invariant causal prediction (ICP) \citep{peters2016icp} identifies parents of a target by exploiting the stability of the conditional distribution $P(Y_i\mid X_{\mathrm{Pa}(i)})$ across environments.
Nonlinear extensions have been developed by \citet{heinze2018nonlinearicp}.
We adapt ICP ideas locally to estimate parent sets and thereby distance-to-intervention, without attempting full graph recovery.

\textbf{Perturbation biology and causal gene networks.}
Single-cell perturbation screens provide a natural setting for interventional causal inference.
Perturb-seq \citep{dixit2016perturbseq} combines CRISPR perturbations with single-cell RNA-seq readouts.
\citet{replogle2022genomewide} scaled this to genome-wide screens.
\citet{peidli2024scperturb} provide harmonized perturbation datasets across studies.
CausalBench \citep{chevalley2023causalbench} benchmarks gene network inference from perturbation data.
Differential expression testing in perturbation screens raises statistical challenges, including the need for appropriate aggregation and multiple testing correction \citep{squair2021confronting,barry2024sceptre}.
Differentially expressed gene (DEG) sets from perturbation screens serve as the inputs for our descendant discovery algorithm (Algorithm~\ref{alg:descendant}).
{\citet{zhu2025auprc} argue that in-silico perturbation models should be evaluated not only by global response accuracy but also by their ability to identify differentially expressed genes, proposing AUPRC as a biologically relevant metric for DE-gene recovery. This motivates uncertainty-quantified post-intervention predictions, since downstream biological interpretation often depends on confidence in DE calls rather than on point predictions alone.}

\section{Interventional Prediction and Selective Conformal}\label{sec:setup}

We consider a collection of {observed} interventions (environments) indexed by $a\in\calA$.
{For split conformal prediction, we partition these observed interventions into disjoint training and calibration sets, $\calA=\calA_{\mathrm{train}}\cup\calA_{\mathrm{cal}}$ and $\calA_{\mathrm{train}}\cap\calA_{\mathrm{cal}}=\emptyset$.}
For each {observed} intervention $a$ we observe i.i.d.\ samples of outcomes $Y^{(a)}\in\R^p$ (e.g., gene expression vector) and possibly covariates $X^{(a)}$.
Typically $X$ represents shared sample-level information that does not depend on $a$ (e.g., cell type, batch indicator, or technical covariates), though in some designs $X^{(a)}$ may include intervention-specific features such as dosage or timing.
We focus on predicting a scalar target component $Y_i^{(a)}$ (or a derived effect size) for each $i\in[p]$, e.g., the expression of a single gene $i$ under intervention $a$.

Let $\hat f_i(a,X)$ be a fitted predictor of $Y_i$ in intervention $a${, learned by an arbitrary predictor-learning algorithm $\calL_i$ from the training interventions $\calA_{\mathrm{train}}$}.
We use a calibration score (nonconformity score)
\[
R_i^{(a)} = S_i\big(X^{(a)}, Y_i^{(a)}\big) \equiv \big|Y_i^{(a)}-\hat f_i(a,X^{(a)})\big|.
\]
{Calibration scores are computed only for interventions $a\in\calA_{\mathrm{cal}}$. The test intervention satisfies $a^\star\notin\calA$ and plays the same role as a new test point in ordinary split conformal prediction.}
{For example, if $\calA_{\mathrm{train}}=\{a_1,a_2\}$, $\calA_{\mathrm{cal}}=\{a_3,a_4,a_5\}$, and $a^\star=a_6$, then $\hat f_i$ is learned using $a_1,a_2$, conformal scores are computed only from $a_3,a_4,a_5$, and $a_6$ is the held-out test intervention.}

{\textbf{Running genomics example.} In a single-cell perturbation screen, an intervention label $a$ denotes a perturbation condition, such as CRISPR interference (CRISPRi) knockdown, CRISPR activation (CRISPRa), or CRISPR-Cas9 knockout of a gene; the perturbation modality can be included as part of the label $a$. For target gene $i$, $Y_i^{(a)}$ is the expression of gene $i$ measured after applying intervention $a$, and $X^{(a)}$ contains sample-level covariates such as cell type, batch, donor, time point, or dose. The prediction model $\hat f_i(a,x)$ answers the question: if we apply intervention $a$ in context $x$, what expression should we expect for gene $i$? It is learned from prior interventional data in $\calA_{\mathrm{train}}$, while $\calA_{\mathrm{cal}}$ is reserved for estimating the residual scale $R_i^{(a)}=|Y_i^{(a)}-\hat f_i(a,X^{(a)})|$. The new intervention $a^\star$ is then a held-out perturbation condition, for example a CRISPRi knockdown not used for either training or calibration. If we fix a cell type or context $X=c$, the notation reduces to asking how the predicted expression of gene $i$ changes as only the perturbation label $a$ varies. Our theory and experiments treat $a$ as a single intervention condition; extending the descendant-learning and coverage analysis to combinatorial perturbations, such as knockdown of two genes, is left for future work.}

\subsection{Selective invariance sets}
Suppose that for each intervention $a{\in\calA\cup\{a^\star\}}$ and target $i$ there exists an unknown indicator
\(
Z_{a,i} = \1\{i \in \desc(a)\},
\)
where $\desc(a)$ denotes causal descendants of $a$ in an underlying causal graph $G=(V,E)$. We assume $a \in \desc(a)$, i.e., $Z_{a,a}=1$. 
We assume a single graph $G$ shared across all samples; extending to covariate-dependent graphs (e.g., cell-type-specific regulatory networks) is an important direction for future work.
Define the \emph{safe calibration set} for target $i$ as $\calA^\star(i)=\{a\in{\calA_{\mathrm{cal}}}: Z_{a,i}=0\}$ (i.e., {calibration} interventions that do not affect $i$).

\begin{assumption}[Selective exchangeability]\label{ass:selective_exchange}
Fix a target gene $i$ and a test intervention $a^\star$ with $Z_{a^\star,i}=0$ (i.e., $i$ is unaffected by $a^\star$).
The calibration scores
$\{R_i^{(a)}:a\in\calA^\star(i)\}\cup\{R_i^{(a^\star)}\}$ are exchangeable.
\end{assumption}

If $Z_{a^\star,i}=0$, we calibrate using interventions where $i$ is unaffected, yielding valid, typically tighter intervals than a pooled baseline mixing {interventions that do and do not affect target $i$}.
{Pooling is justified only for a coarser mixture-level target, e.g., when the test intervention is drawn exchangeably from the same mixture, or when pooled scores are otherwise exchangeable.}
When $Z_{a^\star,i}=1$, selective calibration on the {safe} set does not apply; {one needs an exchangeable affected stratum, or else pooled CP requires an explicit mixture-exchangeability assumption.}
\textbf{{Our framework focuses on $Z_{a^\star,i}=0$, where selective calibration offers the clearest gains; descendant-target strata are left for future work.}}

\begin{example}[Why selective calibration tightens intervals]
Consider a gene regulatory network with 100 genes and 50 interventions.
For a target gene $i$, suppose only 5 of the 50 interventions are ancestors of $i$ (i.e., affect $i$).
Pooled conformal prediction uses all 50 residuals, mixing 5 {residuals from interventions that affect $i$} with 45 {residuals from safe interventions.}
Selective conformal uses only the 45 {safe} residuals.
{When the residuals from interventions that affect $i$ are larger, as can occur for predictors that do not fully model intervention effects,} they inflate the pooled quantile; the selective interval is then tighter while maintaining exact $1-\alpha$ coverage.
\end{example}

\subsection{Selective split conformal interval (oracle)}
Given a set of calibration scores $\{R_i^{(a)}:a\in\calA^\star(i)\}$ of size $m$, define the oracle conformal quantile
\(
k=\left\lceil (m+1)(1-\alpha)\right\rceil
\) as: 
\begin{align*}
\hat q^{\mathrm{oracle}}_{i}(1-\alpha)
= \text{the $k$-th smallest value in } \\
\{R_i^{(a)}:a\in\calA^\star(i)\}\cup\{\infty\}.
\end{align*}
The prediction interval is
\[
\begin{aligned}
C^{\mathrm{oracle}}_{i,1-\alpha}(a^\star,x)
&=[\,\hat f_i(a^\star,x)-\hat q^{\mathrm{oracle}}_{i}(1-\alpha),\\
&\qquad \hat f_i(a^\star,x)+\hat q^{\mathrm{oracle}}_{i}(1-\alpha)\,].
\end{aligned}
\]
Under Assumption~\ref{ass:selective_exchange}, this interval achieves $1-\alpha$ marginal coverage by the standard conformal argument \citep{shafer2008tutorial}.

\section{{Partial Causal Learning for Calibration}}\label{sec:taskdriven}

{The oracle selective interval of Section~\ref{sec:setup} assumes the exchangeable set $\calA^\star(i)$ is known. In practice it is unknown and must be estimated. Rather than learn the full causal graph, we estimate only \textit{the binary descendant indicators} needed to select calibration interventions, and keep the resulting selection error small. Section~\ref{sec:delta} then quantifies exactly how that error translates into a coverage cost, closing the loop.}

\subsection{Binary classification view}
For each intervention $a$ and target $i$, we estimate $\widehat Z_{a,i} \approx Z_{a,i}=\1\{i\in\desc(a)\}$.
We then select calibration interventions for $(i,a^\star)$ as $\widehat{\calA}(i,a^\star) = \{a\in\calA_{\mathrm{cal}}:\ \widehat Z_{a,i}=0\}$, possibly intersected with {design strata within which exchangeability is more plausible} (for example, the {same cell type, experimental context, or batch}).

\subsection{From classifier errors to contamination}
Let $\pi_0$ be the fraction of truly {safe} interventions in $\calA_{\mathrm{cal}}$ for target $i${, meaning interventions that do not affect $i$}.
We use the \emph{descendant-positive} convention throughout, so a ``positive'' is {a pair where intervention $a$ affects target $i$}, i.e., $Z_{a,i}=\1\{i\in\desc(a)\}=1$. Let $\FNR_i=\Pp(\widehat Z_{a,i}=0\mid Z_{a,i}=1)$ be the fraction of {calibration interventions that truly affect $i$ but are misclassified as safe} (the error that \emph{contaminates} the selected calibration set), and let $\FPR_i=\Pp(\widehat Z_{a,i}=1\mid Z_{a,i}=0)$ be the fraction of truly {safe} interventions misclassified as {affecting $i$}, which only reduces the calibration set size. Then the contamination fraction $\delta$ (the share of the selected set that is not exchangeable, made precise in Section~\ref{sec:delta}) satisfies
\[
\delta = \frac{(1-\pi_0)\FNR_i}{\pi_0(1-\FPR_i)+(1-\pi_0)\FNR_i}.
\]
This expression shows that \emph{controlling the false-negative rate $\FNR_i$ is critical}: a missed descendant admits {a non-exchangeable score from an intervention that affects $i$} into the calibration set. A classifier that conservatively labels borderline cases as {``affecting $i$''} (accepting higher $\FPR_i$, which only reduces the calibration set size) is preferable to one that aggressively labels cases as {safe} (low $\FPR_i$ but higher $\FNR_i$, which increases contamination).
{For the hard safe-set selector above, this contamination fraction is primarily target-specific, so one may write $\delta_i$; we keep the notation $\delta$ after fixing the query. It becomes genuinely $(i,a^\star)$-specific when the selected calibration set also depends on the test intervention, for example through same-context restrictions or distance-based weighting.}

\begin{remark}[Asymmetric error costs]
In sparse networks, most interventions do not affect a given target ($\pi_0$ is large), so even a moderate {$\FNR_i$} can lead to small $\delta$.
Indeed, when $\pi_0$ is sufficiently large, a naive strategy of labeling \emph{all} calibration interventions as {safe} (i.e., {$\FNR_i=1$, $\FPR_i=0$}) yields $\delta=(1-\pi_0)/1 = 1-\pi_0$; for example, with $\pi_0\ge 0.95$ this gives $\delta\le 0.05$, requiring no learning algorithm at all.
Algorithm~\ref{alg:descendant} improves upon this baseline by further reducing $\delta$ while maintaining calibration set size.
However, $\delta$ grows rapidly with {$\FNR_i$} when $\pi_0$ is small (dense networks or hub genes).
{This partial-label formulation makes this tradeoff explicit and motivates conservative classification strategies.}
\end{remark}

\subsection{Reduction in complexity}
Full causal graph learning over $p$ variables requires estimating $O(p^2)$ edge parameters (or $O(2^p)$ DAG structures in the worst case).
{This objective reduces full graph learning to estimating $|\calA|\times p$ binary labels, and in practice only a small fraction of these are needed (those for which the corresponding $(a^\star,i)$ pairs will be queried at test time).}
Furthermore, the labels $Z_{a,i}$ have combinatorial structure that can be exploited: if $a\to b\to c$ in $G$, then $Z_{a,c}=1$ and $Z_{b,c}=1$, enabling information sharing across interventions.

\section{\texorpdfstring{$\delta$}{delta}-Robustness: Coverage under Contaminated Selective Calibration}\label{sec:delta}

{{The partial-label selector of Section~\ref{sec:taskdriven} produces the estimated calibration set} $\widehat{\calA}(i,a^\star)=\{a\in\calA_{\mathrm{cal}}:\widehat Z_{a,i}=0\}$ (recall {$\calA_{\mathrm{cal}}$ is the held-out calibration split, with $a^\star\notin\calA_{\mathrm{cal}}$}). Because $\widehat Z$ is imperfect, some selected interventions violate exchangeability ({they truly affect target $i$ but are misclassified as safe}), contaminating the calibration scores. We now quantify the coverage cost of this contamination as a function of its magnitude.}

\begin{definition}[Contamination fraction]
Let $\calA^\star=\calA^\star(i)$ be the (unknown) set of truly exchangeable {safe} calibration interventions for target $i$.
Define the contamination fraction
\[
\delta \equiv \frac{|\widehat{\calA}(i,a^\star)\setminus \calA^\star|}{|\widehat{\calA}(i,a^\star)|}.
\]
That is, $\delta$ is the fraction of calibration interventions in the selected set that actually {affect target $i$} (i.e., $Z_{a,i}=1$) and therefore not exchangeable with the test score.
\end{definition}

The selected set $\widehat{\calA}$ naturally partitions into three regions (Figure~\ref{fig:sets}):
\emph{good} interventions $\widehat{\calA}\cap\calA^\star$ ($m$ truly safe and exchangeable with $a^\star$);
\emph{bad} interventions $\widehat{\calA}\setminus\calA^\star$ ($n-m$ contaminants that break exchangeability);
and \emph{missed} interventions $\calA^\star\setminus\widehat{\calA}$ (truly safe but not selected{:} wasted data that reduces calibration set size but does not harm coverage).
The contamination fraction is $\delta=(n-m)/n$.

\begin{figure*}[t]
\centering
\begin{minipage}[t]{0.48\textwidth}
\centering
\textbf{(a)} Partition of calibration interventions\\[6pt]
\begin{tikzpicture}[scale=0.56, every node/.style={font=\scriptsize}]
  \draw[thick, rounded corners=2pt, fill=ltbg]
    (-0.2,-1.7) rectangle (7.2,2.9)
    node[above left, font=\scriptsize\bfseries] {$\calA_{\mathrm{cal}}$};

  \fill[safegrn, fill opacity=0.20]
    (2.1,0.5) ellipse (2.1cm and 1.5cm);
  \draw[safegrn, thick]
    (2.1,0.5) ellipse (2.1cm and 1.5cm);
  \node[safegrn, font=\scriptsize\bfseries] at (0.5,1.85)
    {$\calA^\star\!(i)$};

  \fill[selblue, fill opacity=0.15]
    (4.5,0.5) ellipse (2.1cm and 1.5cm);
  \draw[selblue, thick]
    (4.5,0.5) ellipse (2.1cm and 1.5cm);
  \node[selblue, font=\scriptsize\bfseries] at (6.0,1.85)
    {$\widehat{\calA}$};

  \node[safegrn, font=\scriptsize, align=center] at (3.3,0.8)
    {\textbf{Good}\\[-1pt]($m$)};
  \node[safegrn, font=\tiny] at (3.3,0.0) {exchangeable};

  \node[badred, font=\scriptsize, align=center] at (5.7,0.5)
    {\textbf{Bad}\\[-1pt]($n\!-\!m$)};

  \node[missorg, font=\scriptsize] at (0.7,0.5) {\textbf{Missed}};
  \node[missorg, font=\tiny] at (0.7,-0.15) {{FP}};

  \node[circle, draw=badred, very thick, fill=ltred,
        minimum size=5mm, font=\scriptsize\bfseries]
    (astar) at (8.3,0.5) {$a^\star$};
  \node[font=\tiny, below=1pt of astar] {test};

  \draw[decorate, decoration={brace, amplitude=3pt, mirror}]
    (3.9,-1.4) -- (6.5,-1.4)
    node[midway, below=4pt, font=\scriptsize]
    {$\delta = \tfrac{n-m}{n}$};
\end{tikzpicture}
\end{minipage}\hfill
\begin{minipage}[t]{0.48\textwidth}
\centering
\textbf{(b)} Concrete example on a causal DAG\\[6pt]
\begin{tikzpicture}[
    gene/.style={circle, draw, thick, minimum size=5.5mm,
                 font=\scriptsize\bfseries, inner sep=0pt},
    act/.style={-{Stealth[length=1.5mm]}, thick},
    node distance=8mm
  ]
  \node[gene, fill=ltbg]      (a1) at (0,2.0)   {$a_1$};
  \node[gene, fill=ltgrn]     (a2) at (1.2,2.0) {$a_2$};
  \node[gene, fill=ltgrn]     (a3) at (2.4,2.0) {$a_3$};
  \node[gene, fill=ltred]     (a4) at (3.6,2.0) {$a_4$};
  \node[gene, fill=orange!25, draw=missorg, thick]
                               (a5) at (4.8,2.0) {$a_5$};

  \node[gene, fill=ltbg]      (g) at (0.6,0.8) {$g$};
  \node[gene, fill=ltbg]      (h) at (3.0,0.8) {$h$};
  \node[gene, fill=ltbg]      (k) at (4.8,0.8) {$k$};

  \node[gene, fill=ltblu, very thick, draw=selblue]
    (i) at (1.8,-0.3) {$i$};

  \node[gene, fill=ltred!50, draw=badred, very thick]
    (as) at (4.8,-0.3) {$a^\star$};

  \draw[act] (a1) -- (g);
  \draw[act] (g)  -- (i);
  \draw[act] (a2) -- (h);
  \draw[act] (a3) -- (h);
  \draw[act] (a4) -- (g);
  \draw[act] (a4) -- (h);
  \draw[act] (a5) -- (k);
  \draw[act] (as) -- (k);

  \node[font=\tiny, selblue, below=1pt of i] {target};
  \node[font=\tiny, badred, below=1pt of as] {test};
\end{tikzpicture}
\end{minipage}
\caption{{Intervention-set geometry for target $i$.
\textbf{(a)} Within the calibration split, $\calA^\star(i)$ is the truly safe set and $\widehat{\calA}$ is the selected set; their overlap gives good scores, contaminants, and missed safe scores.
Contaminants drive $\delta=(n-m)/n$ and reduce coverage, while missed safe scores only reduce calibration size.
\textbf{(b)} Example where $\widehat{\calA}=\{a_2,a_3,a_4\}$: $a_2,a_3$ are good, $a_4$ is a contaminant, and $a_5$ is missed.}}
\label{fig:sets}
\end{figure*}

\subsection{Main robustness theorem}
We prove a finite-sample coverage lower bound that depends explicitly on $\delta$ and does \emph{not} assume anything about the score distribution on contaminated interventions.

\begin{theorem}[$\delta$-robust selective conformal coverage]\label{thm:delta}
Fix $(i,a^\star)$ with $Z_{a^\star,i}=0$, and suppose that scores from the ``good'' set
$\calA^\star$ are exchangeable with the test score $R_i^{(a^\star)}$ (Assumption~\ref{ass:selective_exchange}).
Let $\widehat{\calA}$ be any selected calibration set of size $n=|\widehat{\calA}|$, and let $m=|\widehat{\calA}\cap\calA^\star|$ be the number of good calibration interventions.
Construct the split conformal interval using all scores in $\widehat{\calA}$:
let $\hat q$ be the $\lceil (n+1)(1-\alpha)\rceil$-th order statistic of $\{R_i^{(a)}:a\in\widehat{\calA}\}{\cup\{\infty\}}$ and output $C_{i,1-\alpha}=[\hat f_i\pm \hat q]$.
Then, for all possible distributions of contaminated scores,
\[
\Pp\big(R_i^{(a^\star)} \le \hat q\big)
\ \ge\
1-\alpha
-\frac{n-m}{m+1}.
\]
In terms of the contamination fraction $\delta=(n-m)/n$,
\[
\begin{aligned}
\Pp\big(Y_i^{(a^\star)} \in C_{i,1-\alpha}\big)
&\ge 1-\alpha - g(\delta,n),
\end{aligned}
\]where $g(\delta,n) \equiv \frac{\delta n}{(1-\delta)n+1}$ \& 
$\lim_{n\to\infty} g(\delta,n)\approx \delta/(1-\delta)$.
\end{theorem}

\begin{remark}[Interpretation and how to use it]
Theorem~\ref{thm:delta} turns structure learning error into a direct statistical price for selective inference.
If $\delta$ is small, coverage is close to nominal.
{The same bound can be used prospectively: one can design the learning procedure to keep $\delta$ below a threshold $\delta_{\max}$ such that $g(\delta_{\max},n)\le \epsilon$ for a desired coverage slack $\epsilon$.}
{The formal $\alpha$-correction based on an upper bound $\hat\delta$ is stated separately in Corollary~\ref{cor:corrected}.}
\end{remark}

\begin{remark}[Comparison with Huber contamination bounds]
{Unlike the Huber model of \citet{clarkson2024contamination}, where an $\epsilon$-fraction of \emph{all} calibration points is corrupted, here contamination arises from misclassification of the selective/Mondrian stratum: a calibration intervention that truly belongs outside the safe stratum for target $i$ is incorrectly admitted into $\widehat{\calA}$. For the hard safe-set selector, $\delta$ is target-specific (depends on $i$ given the split); it becomes $(i,a^\star)$-specific only if selection also uses test-dependent strata such as cell type, batch, or distance to $a^\star$. Exploiting this selective structure yields a bound tailored to the actual selected calibration set rather than a generic Huber corruption model.}
\end{remark}

\begin{remark}[Benign contamination]
In many perturbation applications, interventions misclassified as {safe} will actually induce \emph{larger} residuals (because the intervention does change the target), making $\hat q$ larger and the interval conservative.
Theorem~\ref{thm:delta} is worst-case over contaminating distributions; empirical coverage is often better than the bound.
{
Corollary~\ref{cor:corrected} below concerns how to correct the nominal conformal level once an upper bound on $\delta$ is available.}
\end{remark}

\begin{corollary}[Corrected selective conformal]\label{cor:corrected}
If $\hat\delta$ is an upper bound on the contamination fraction $\delta$ (possibly estimated or given by the recovery conditions of Propositions~\ref{prop:superset}--\ref{prop:fp}), then running split conformal with nominal level $\alpha' = \alpha - g(\hat\delta,n)$ yields coverage at least $1-\alpha$ for the selective interval (when $\alpha'\le 0$, the interval is $(-\infty,\infty)$ and coverage is trivially $1$).
\end{corollary}

\section{Algorithms}\label{sec:algorithms}

We present two practical estimators from interventional data: one for $\widehat Z_{a,i}$ and one for {a distance-to-intervention score $\widehat d(a,i)$, defined below as an estimated path length from intervention $a$ to target $i$}.
Both are designed for scalability in high dimensions and many interventions.

\subsection{Descendant discovery via perturbation intersection patterns}

The first algorithm estimates descendant sets $\widehat{\desc}(a)$ from differentially affected variable sets across interventions, using set intersections to prune false positives.

\textbf{Input: differentially affected sets.}
Assume we have a differentially affected set $\calS_a\subset[p]$ for each intervention $a\in\calA${,} that is, an estimate of the variables whose distribution changes under $a$.
In genomics these are differentially expressed gene (DEG) sets, computed by standard statistical tests (e.g., two-sample $t$-tests or Wilcoxon rank-sum tests with Benjamini--Hochberg correction \citep{squair2021confronting,barry2024sceptre}); in other domains any test for distributional change applies.
$\calS_a$ is an estimate of $\{i: Z_{a,i}=1\}=\desc(a)$, but may contain false positives and false negatives.

\textbf{Intersection estimator.}
For each intervention $a$, identify interventions that are ``upstream'' of $a$, $\calU(a) = \{b\in\calA : a\in\calS_b\}$ (i.e., interventions $b$ such that $a$ is affected by $b$).
We estimate the descendant set of $a$ by intersecting its own affected set with those of its upstream interventions:
\[
\widehat{\desc}(a)=
\begin{cases}
\calS_a, & \calU(a)=\emptyset,\\[3pt]
\calS_a\cap \bigcap_{b\in\calU(a)} \calS_b, & \text{otherwise.}
\end{cases}
\]
The intuition is: if $b$ is upstream of $a$ (i.e., $a\in\desc(b)$), then every descendant of $a$ is also a descendant of $b$, so $\desc(a)\subseteq\desc(b)$.
Thus descendants of $a$ should appear consistently across $\calS_a$ and across upstream affected sets, while spurious entries will not.
Equivalently, one can view this as starting from $\calS_a$ and pruning any candidate variable that fails to appear in some upstream affected set; we present the set-intersection form for clarity.

\textbf{Complexity.}
Algorithm~\ref{alg:descendant} runs in $O(|\calA|^2 \cdot p)$ time in the worst case (for each of $|\calA|$ interventions, iterating over $O(|\calA|)$ upstream interventions and $O(p)$ candidate genes).
In practice, the estimator is implemented with set operations; since affected sets are typically sparse, the intersections are much cheaper than the worst-case bound.
{The expected cost is $O(|\calA|^2\,\bar s)$, where $\bar s=\E[|\calS_a|]$ is the expected affected-set size; under the Replogle top-$10\%$ rule $\bar s\approx 500$ against $p\approx 5{,}000$, roughly a $10\times$ reduction over the worst case.}

\subsection{Local ICP for distance estimation}\label{sec:alg_icp}

Beyond binary descendant labels, selective inference can benefit from a \emph{distance-to-intervention} notion $\hat d(a,i)$ that prioritizes closer calibration interventions or enables weighted conformal calibration.
We propose a local search inspired by ICP \citep{peters2016icp} that traces back from the target $i$ through estimated parents to obtain a coarse path-length estimate $\widehat d(a,i)$ without learning the full graph.
The full algorithm (Algorithm~\ref{alg:local_icp}) and details on using distance for kernel-weighted calibration are given in the Supplementary Material.
{The parent-pool construction and runtime accounting for Local ICP are deferred to Appendix~\ref{app:alg_icp}.}

\begin{algorithm}[t]
\caption{Descendant Discovery via Differentially Affected Set Intersections}
\label{alg:descendant}
\begin{algorithmic}[1]
\footnotesize
\REQUIRE Differentially affected sets $\{\calS_a\}_{a\in\calA}$, interventions $\calA$
\ENSURE Estimated descendant sets $\{\widehat{\desc}(a)\}_{a\in\calA}$
\FOR{each intervention $a\in\calA$}
  \STATE $\calU(a) \gets \{b\in\calA: a\in\calS_b\}$ \COMMENT{Upstream interventions of $a$}
  \IF{$\calU(a) = \emptyset$}
    \STATE $\Cand(a) \gets \calS_a$ \COMMENT{Fallback to own affected set}
  \ELSE
    \STATE $\Cand(a) \gets \calS_a \cap \bigcap_{b\in\calU(a)} \calS_b$ \COMMENT{Intersection with upstream affected sets}
  \ENDIF
  \STATE $\widehat{\desc}(a) \gets \Cand(a)$
\ENDFOR
\RETURN $\{\widehat{\desc}(a)\}_{a\in\calA}$
\end{algorithmic}
\end{algorithm}

\section{Conditions for Recovery}\label{sec:recovery}

This section states concrete, checkable conditions under which Algorithm~\ref{alg:descendant} controls the contamination fraction $\delta$.

\begin{assumption}[Interventional faithfulness and detectability]\label{ass:mono}
There is an underlying causal DAG $G=(V,E)$ such that intervening on node $a$ changes (in distribution) precisely the descendants of $a${,} i.e., for every $i\in\desc(a)$ the marginal $P(Y_i\mid \mathrm{do}(a))$ differs from $P(Y_i)$ (no path cancellations), and for every $i\notin\desc(a)$ it is unchanged.
Moreover, each descendant effect is detectable with probability at least $1-\epsilon_{\mathrm{fn}}$ by the testing procedure used to form $\calS_a$.
Formally: for each $a\in\calA$ and $i\in\desc(a)$,
$\Pp(i\in\calS_a)\ge 1-\epsilon_{\mathrm{fn}}$,
and for each $i\notin\desc(a)$,
$\Pp(i\in\calS_a)\le \epsilon_{\mathrm{fp}}$.
\end{assumption}

\begin{proposition}[{Estimated descendant sets contain the true descendants}]\label{prop:superset}
Under Assumption~\ref{ass:mono}, for any intervention $a\in\calA$, if $\calU(a)\neq\emptyset$ and for every $b\in\calU(a)$ we have $a\in\desc(b)$ (i.e., all identified upstream interventions are true ancestors), then with probability at least $1-|\desc(a)|\cdot(|\calU(a)|+1)\cdot\epsilon_{\mathrm{fn}}$ (by a union bound),
\(
\desc(a)\subseteq {\widehat{\desc}(a)}.
\)
\end{proposition}

\begin{assumption}[Upstream diversity for intersections]\label{ass:diversity}
For any intervention $a\in\calA$ and non-descendant $i\notin \desc(a)$, there exists an intervention $b\in\calA$ such that $a\in\desc(b)$ but $i\notin\desc(b)$.
Moreover, for such a $b$ the testing procedure satisfies
$\Pp(a\in\calS_b \text{ and } i\notin\calS_b)\ge 1-\epsilon_{\mathrm{cx}}$.
\end{assumption}

Assumption~\ref{ass:diversity} requires ``upstream diversity'': for each non-descendant $i$, there exists an ancestor of $a$ whose descendants include $a$ but not $i$.
This is plausible when ancestors of $a$ have sufficiently diverse descendant sets.
{Here $\epsilon_{\mathrm{cx}}$ is a cross-screening failure probability: for the auxiliary intervention $b$, it is the chance that either $a$ is not detected in $\calS_b$ or the non-descendant $i$ is spuriously included in $\calS_b$. Under Assumption~\ref{ass:mono}, a union bound gives $\epsilon_{\mathrm{cx}}\le \epsilon_{\mathrm{fn}}+\epsilon_{\mathrm{fp}}$ for any fixed $b$ satisfying the structural condition; we keep $\epsilon_{\mathrm{cx}}$ separate because the joint event can be estimated or bounded directly. The final false-positive probability for the estimated descendant set $\widehat{\desc}(a)$ is bounded by both $\epsilon_{\mathrm{fp}}$ and $\epsilon_{\mathrm{cx}}$, so we write $\epsilon_{\mathrm{sel}}=\min\{\epsilon_{\mathrm{fp}},\epsilon_{\mathrm{cx}}\}$ for the sharper of these two bounds.}

\begin{proposition}[False positive control via upstream intersections]\label{prop:fp}
Under Assumptions~\ref{ass:mono}--\ref{ass:diversity}, for any non-descendant $i\notin\desc(a)$, the intersection estimator in Algorithm~\ref{alg:descendant} excludes $i$ with probability at least {$1-\epsilon_{\mathrm{sel}}$}.
Consequently, the false positive rate 
satisfies
\(
\Pp(\widehat Z_{a,i}=1 \mid Z_{a,i}=0) \le {\epsilon_{\mathrm{sel}}}.
\)
\end{proposition}

\begin{corollary}[Contamination control via recovery conditions]\label{cor:contamination}
Under Assumptions~\ref{ass:mono}--\ref{ass:diversity} {and the upstream-ancestor condition of Proposition~\ref{prop:superset}}, the expected contamination fraction for target $i$ and the calibration set $\widehat{\calA}(i,a^\star)$ satisfies
{\[
\E[\delta] \le \frac{(1-\pi_0)(\bar u+1)\epsilon_{\mathrm{fn}}}{\pi_0(1-{\epsilon_{\mathrm{sel}}})+(1-\pi_0)(\bar u+1)\epsilon_{\mathrm{fn}}},
\]}
where $\pi_0$ is the fraction of {safe} calibration interventions for target $i$ and $\bar u = \max_a |\calU(a)|$ is the maximum upstream set size across interventions{; the factor $\bar u+1$ counts the own affected set $\calS_a$ together with the $|\calU(a)|$ upstream sets in the union bound}.
{Under the descendant-positive convention the contamination-driving false-negative rate $(\bar u+1)\epsilon_{\mathrm{fn}}$ appears in the numerator and the efficiency-cost false-positive rate {$\epsilon_{\mathrm{sel}}$} in the denominator; a full derivation is in Appendix~\ref{app:fpr_fnr}.}
In sparse networks where $\pi_0$ is near 1 and {$\epsilon_{\mathrm{fn}}$} is small, $\E[\delta]$ is small.
\end{corollary}

Together, Propositions~\ref{prop:superset}--\ref{prop:fp} and Corollary~\ref{cor:contamination} imply a tunable tradeoff between {spurious descendants (false positives, which reduce calibration set size but do not violate coverage) and missed descendants (false negatives, which contaminate the calibration set)}, directly controlling the coverage loss via Theorem~\ref{thm:delta}.

{Assumption~\ref{ass:diversity} (upstream diversity) is the strongest of the three, but Algorithm~\ref{alg:descendant} degrades gracefully when it is only partially satisfied. When no upstream interventions are identified for $a$ (i.e., $\calU(a)=\emptyset$), the algorithm sets $\widehat{\desc}(a)=\calS_a$ rather than failing, so Theorem~\ref{thm:delta} remains in force: a missing intersection step raises the {required contamination upper bound $\hat\delta$} (and hence the width of the corrected interval) rather than silently breaking coverage. The false-positive control of Proposition~\ref{prop:fp} weakens for such $a$, but validity is preserved through the $\hat\delta$-correction of Corollary~\ref{cor:corrected}.}
{The FNR bound is conditional on the identified upstream interventions being true ancestors; if spurious upstream interventions enter $\calU(a)$, they can prune true descendants, so their probability must be included in any unconditional end-to-end recovery bound.}

\section{Experiments}\label{sec:experiments}

We evaluate the framework on synthetic and real data, focusing on three questions:
(Q1) Does selective conformal prediction with learned $\widehat Z$ maintain valid coverage?
(Q2) Does coverage degrade with contamination $\delta$ as predicted by Theorem~\ref{thm:delta}?
(Q3) Does the corrected procedure (Corollary~\ref{cor:corrected}) restore nominal coverage {when supplied with realized or proxy contamination}?
{In the experiments, Corrected uses the realized contamination fraction computed from true synthetic labels or real-data proxy labels, so these results test the correction formula given such a bound rather than a fully deployable $\hat\delta$ estimator.}

\noindent
Throughout, \emph{Coverage} denotes the empirical fraction of test points whose true outcome falls within the prediction interval, and \emph{Width} denotes the average interval length across test points.
Both are computed over held-out (test intervention, target gene) pairs {with $Z_{a^\star,i}=0$, i.e.\ pairs the test intervention does not affect, for which a selective interval is meaningful}.
{We treat coverage as a validity constraint rather than a monotone leaderboard score: values above $0.9$ are acceptable, but higher coverage can simply mean more conservative intervals. In tables, bold coverage values meet or exceed the nominal $0.9$ level. Among entries that meet this coverage threshold, bold width marks the narrowest interval and underlined width marks the second narrowest; widths for methods below nominal coverage are not treated as wins. Appendix~\ref{app:method_definitions} gives procedural definitions of the table methods and additional baselines.}

\subsection{Synthetic interventional SEM}\label{sec:exp_sem}

\textbf{Setup.}
We generate random DAGs using the Erd\H{o}s--R\'enyi model on $p=200$ nodes with expected degree $d_{\mathrm{avg}}=2.0$.
Edge weights are sampled uniformly from $[-1,-0.3]\cup[0.3,1]$.
The linear Gaussian SEM is $V = B^\top V + \varepsilon$ with $\varepsilon\sim\mathcal{N}(0,I)$, where $V\in\R^p$ is the vector of all node values.
Interventions are hard (do-operator), setting $V_a:=0$.
We use $n_{\mathrm{obs}}=200$ observational and $n_a=200$ interventional samples per node, with $|\calA|=150$ randomly selected intervention targets.
Ground-truth descendant sets $\desc(a)$ are computed from the transitive closure of $B$.

\textbf{Conformal scores.}
To isolate calibration-set selection from predictor quality, we use synthetic scores: {safe} pairs $(Z_{a,i}=0)$ receive $R_i^{(a)} = |N(0,1)|$, while {pairs where intervention $a$ affects target $i$} receive $R_i^{(a)} = |N(0,0.15)|$.
These smaller contaminating scores pull the conformal quantile downward, causing under-coverage {when the calibration set is contaminated by these affected interventions}.

\textbf{DEG procedure and descendant estimation.}
Gene-wise two-sample $t$-tests comparing interventional vs.\ observational data with Benjamini--Hochberg correction at $q=0.05$ yield DEG sets $\calS_a$, and Algorithm~\ref{alg:descendant} estimates $\widehat Z_{a,i}$.
Interventions are split 10\%/81\%/9\% into training/calibration/test; $\widehat Z$ is fit without test interventions, and coverage is evaluated on held-out tests using calibration interventions. {No outcome predictor is trained here: the training split feeds only descendant estimation (the DEG tests and Algorithm~\ref{alg:descendant}).}
Because synthetic scores are generated independently of DEG-based descendant estimation, selecting $\widehat{\calA}$ does not create dependence between $\widehat Z$ and calibration scores.

\textbf{Results.}
Table~\ref{tab:main} summarizes results over 20 random seeds (220{,}132 evaluations).
All four methods achieve coverage near the nominal $1-\alpha=0.9$ level.
The {realized contamination fraction, computed from the true DAG after Algorithm~\ref{alg:descendant} selects the calibration set,} is very low ($\delta=0.018$), indicating that the intersection-based procedure keeps contamination well-controlled in the linear SEM setting.
Because $\pi_0$ is high in this sparse network (most interventions {do not affect} a given target), the learned selector labels nearly all calibration interventions as {safe}, so $\widehat{\calA}\approx\calA_{\mathrm{cal}}$ and the Estimated and Pooled methods coincide; their differences emerge under controlled contamination (Section~\ref{sec:exp_delta}).
{Corrected selective CP is the only non-oracle method above nominal coverage in this table, {using the realized $\delta$ for the level correction}; wider intervals reflect the cost of the $\alpha$-correction even when contamination is small.}

\begin{table}[t]
\centering
\caption{{Main synthetic experiment ($p=200$, $|\calA|=150$, $\alpha=0.1$). Coverage and width are averaged over 20 seeds. {The realized $\delta$ is not applicable to Pooled, which applies no selection.} 
}}
\label{tab:main}
\footnotesize
\begin{tabular}{lcccc}
\toprule
Method & Coverage & Width & $n_{\mathrm{cal}}$ & {realized $\delta$} \\
\midrule
Oracle & $\okcov{0.901}$ & $\bestw{3.35}$ & $118.8$ & $0.000$ \\
Estimated & $0.899$ & $3.32$ & $121.0$ & $0.018$ \\
Pooled & $0.899$ & $3.32$ & $121.0$ & {--} \\
Corrected & $\okcov{0.918}$ & $\secondw{3.58}$ & $121.0$ & $0.018$ \\
\bottomrule
\end{tabular}
\end{table}

\subsection{Controlled \texorpdfstring{$\delta$}{delta} ablation}\label{sec:exp_delta}

To test Theorem~\ref{thm:delta}, we replace a target fraction $\delta_{\mathrm{inject}}\in\{0, 0.05, 0.1, 0.15, 0.2, 0.3\}$ of scores in the true {safe} calibration set with {scores from interventions that affect the target}, resampling to achieve the target contamination rate.
We use $p=200$ nodes, $|\calA|=150$ interventions, and repeat over 100 random seeds (1{,}065{,}408 total evaluations).
{For finite-width reporting, we clip the corrected level below at $\alpha'=0.01$; if $\alpha-g(\delta,n)\le0$, Corollary~\ref{cor:corrected} would instead return the infinite interval.}

\textbf{Results.}
Figure~\ref{fig:delta_cov} shows the results (full numerical breakdown in Table~\ref{tab:delta} in the Supplementary Material).
Estimated coverage drops monotonically from $0.905$ at $\delta=0$ to $0.867$ at $\delta=0.30$, as predicted by Theorem~\ref{thm:delta}.
{When supplied with the injected contamination level,} the Corrected method empirically achieves coverage $\ge 0.95$ at all nonzero contamination levels ($\delta\ge 0.05$), well above the nominal $0.9$, at the cost of $1.2$--$1.8\times$ wider intervals {under this clipped finite-width implementation}.
Oracle and Pooled coverage remain flat, while empirical coverage always exceeds the theoretical lower bound $1-\alpha-g(\delta,n)$ (Figure~\ref{fig:bound_gap} in the Supplementary Material).

\begin{figure}[t]
\centering
\includegraphics[width=0.85\linewidth]{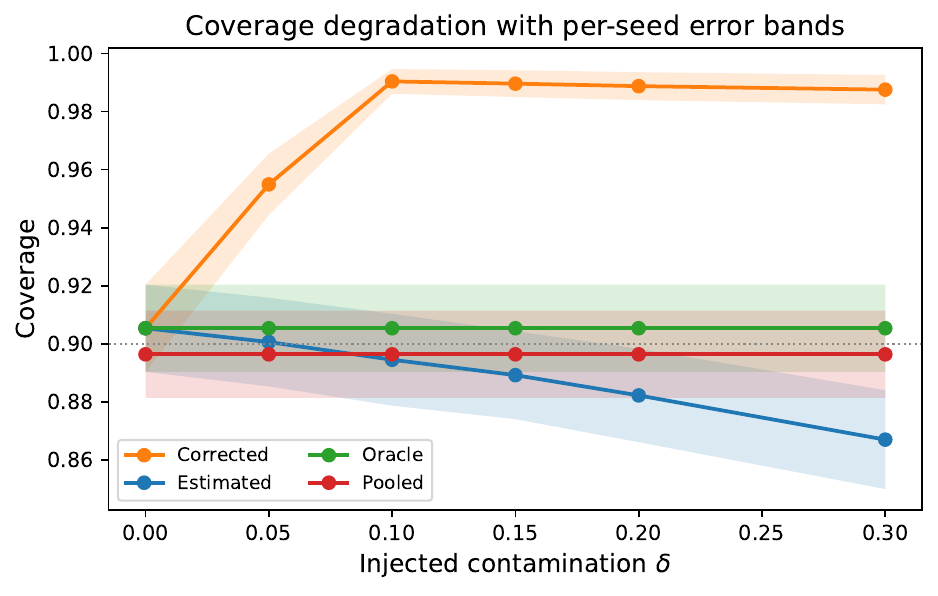}
\caption{{Coverage under injected contamination. Estimated degrades with $\delta$, Corrected remains above nominal {when supplied with the injected $\delta$} under the clipped finite-width implementation, and Oracle/Pooled are flat. Dashed line: $1-\alpha=0.9$; bands show $\pm1$ seed-level standard deviation.}}
\label{fig:delta_cov}
\end{figure}

\subsection{Real perturbation data: Replogle K562 CRISPRi screen}\label{sec:exp_real}

We apply the complete method to real data from the Replogle K562 CRISPRi genome-scale screen~\citep{replogle2022genomewide}, accessed via Zenodo.
We select the 50 perturbations with the largest cell counts ($\ge 200$ cells each), retaining $p\approx 5{,}000$ genes expressed in $\ge 10\%$ of cells.
{Prior perturbation-screen benchmarks often use restricted intervention or gene subsets for scale \citep{sethuraman2023nodags,rohbeck2024bicycle,roohani2024gears,lopez2022dcdfg}, and we follow the same practical convention using a fixed cell-count-selected Replogle subset of 50 perturbations and $\sim$5{,}000 expressed genes.}
Log-fold-change (LFC) vectors are computed per perturbation against the non-targeting control.
Since ground-truth descendant sets are unavailable, we define proxy ``oracle'' affected sets as the top $10\%$ of genes by absolute LFC per perturbation, and DEG sets analogously.
Perturbations are split into 10\% training, 81\% calibration, and 9\% test (5 test perturbations, $\sim$40 calibration perturbations).
We fit $\widehat Z$ using only the training and calibration perturbations (holding out test perturbations from descendant discovery), and evaluate coverage on the held-out test perturbations using calibration perturbations for conformal calibration.

\textbf{Results.}
Table~\ref{tab:real} {reports the primary four methods and additional baselines on the same split; Appendix~\ref{app:method_definitions} defines all table methods procedurally}.
{Among the causal-selector methods, Corrected is the only one above nominal coverage} ($0.906$, averaged over feasible evaluations{, those yielding a finite interval}), {using the proxy contamination fraction derived from LFC-based labels}.
However, it is finite for only $59.8\%$ of evaluations; the rest are infinite because {the clipped corrected level is often $\alpha'=0.01$, requiring roughly $99$ calibration scores, while this split has} $n_{\mathrm{cal}}\approx 40$. 
{The proxy oracle undercovers ($0.864$), indicating proxy safe-set exchangeability violations.}
{{Weighted CP~\citep{tibshirani2019covshift} is implemented leakage-free: weights use only control-cell coexpression of perturbation targets, not held-out LFC response profiles. It reaches $0.925$ coverage with full feasibility but wider intervals ($0.391$), so it trades width for feasibility relative to Corrected.} {The full-graph proxy and observational-only heuristic collapse to Estimated/Pooled behavior ($0.888$), supporting Algorithm~\ref{alg:descendant} as the selector of choice in this sparse regime.
{Corrected is therefore above nominal coverage on the feasible proxy-labeled cases but not dominant on this real-data proxy.}}}

\begin{table}[t]
\centering
\caption{{Real-data results on Replogle K562 CRISPRi. Coverage and width are averaged over all evaluations except Corrected, where they are averaged over feasible finite-interval cases, so its width is not directly comparable to full-feasibility rows. {Corrected uses the proxy contamination fraction derived from LFC-based labels.} Feasible is the fraction of finite intervals; only Corrected produces infinite intervals. For WCP, $n_{\mathrm{cal}}$ is the Kish effective sample size.}}
\label{tab:real}
\footnotesize
\begin{tabular}{lcccc}
\toprule
Method & Coverage & Width & $n_{\mathrm{cal}}$ & Feasible \\
\midrule
Oracle (proxy) & $0.864$ & $0.306$ & $36.7$ & $100\%$ \\
Estimated & $0.888$ & $0.349$ & $40.0$ & $100\%$ \\
Pooled & $0.888$ & $0.349$ & $40.0$ & $100\%$ \\
Corrected & $\okcov{0.906}$ & $0.329$ & $40.0$ & $59.8\%$ \\
{{Ctrl-coexpr WCP}} & {$\okcov{0.925}$} & {$0.391$} & {$40.0$} & {$100\%$} \\
{{Full-graph proxy}} & {$0.888$} & {$0.354$} & {$39.7$} & {$100\%$} \\
{Obs.-only sel.} & {$0.888$} & {$0.349$} & {$40.0$} & {$100\%$} \\
\bottomrule
\end{tabular}
\end{table}

\textbf{Limitations of the real-data evaluation.}
The proxy oracle is constructed from LFC quantiles, not from true causal knowledge.
Its coverage shortfall ($0.864$ vs.\ the nominal $0.9$) indicates that the LFC-based proxy does not perfectly capture the exchangeability structure, likely because (i) indirect downstream effects create correlated residuals even among ``unaffected'' genes, and (ii) batch effects and technical noise break the i.i.d.\ assumption within strata.
{Moreover, on real data the proxy safe sets and residual scores are both computed from the same expression matrix, so the selective-exchangeability assumption should be read as an approximate diagnostic rather than a verified condition.}
{A genuine oracle would require independently validated causal annotations (for example, curated knockout-effect databases or matched perturbation atlases) which are not available at the scale of $50$ perturbations $\times$ $5{,}000$ genes for K562; the controlled-$\delta$ synthetic experiment (Section~\ref{sec:exp_delta}) therefore remains the cleanest validation of the method itself.}
Despite these violations, {within the causal-selector methods,} the qualitative ordering {(}Corrected $>$ Estimated $\approx$ Pooled $>$ Oracle{)} is consistent with the theoretical predictions, and {Corrected is the only one of these methods to exceed nominal coverage}.
{The $59.8\%$ feasibility is a finite-sample artifact of the small calibration set rather than a methodological limit: a scaling experiment on the same split (Appendix~\ref{app:extra}, Figure~\ref{fig:feasibility}) shows feasibility climbing monotonically to $100\%$ by $n_{\mathrm{cal}}=160$ while {proxy $\delta$} stays near $0.083$.}
A more thorough real-data evaluation with larger calibration sets (more perturbations) and biologically validated oracle sets remains open.

\section{Conclusion}\label{sec:discussion}

{We show that selective conformal calibration can exploit partial causal structure without requiring full graph recovery, and that the coverage cost of imperfect selection can be quantified directly through the contamination fraction of the selected calibration set. The synthetic experiments provide the cleanest validation of the theory, while the Replogle K562 experiment illustrates both the promise and the finite-sample limitations of applying the method with proxy causal labels. Open directions are collected in Appendix~\ref{app:future_work}.}

\begin{acknowledgements}
A.A. was partially supported by the Patient-Centered Outcomes Research Institute (PCORI) award ME-2023C1-32148 and by the National Institutes of Health under award R01MH139379. J.P.L. was partially supported by the National Cancer Institute and the National Center for Advancing Translational Sciences of the National Institutes of Health under awards P50CA127001-16, CCSG P30CA016672-46, and CCTS UM1TR004906.

\paragraph{Reproducibility.} All experiments, tables, and figures in this paper can be reproduced with the code, data-processing scripts, and notebooks available at \url{https://github.com/AsiaeeLab/selective-conformal-interventions}.
\end{acknowledgements}

\bibliographystyle{plainnat}
\bibliography{refs}

\newpage
\onecolumn

\title{Partial Causal Structure Learning for Valid Selective Conformal Inference under Interventions\\(Supplementary Material)}
\maketitle

\appendix

\section{Proof of Theorem~\ref{thm:delta}}\label{app:proof}

\begin{proof}[Proof of Theorem~\ref{thm:delta}]
\textbf{Setup.}
Fix target $i$ and test intervention $a^\star$ with $Z_{a^\star,i}=0$.
Let $\widehat{\calA}$ be the selected calibration set of size $n$.
Partition $\widehat{\calA}$ into $G=\widehat{\calA}\cap\calA^\star$ (good, $|G|=m$) and $B=\widehat{\calA}\setminus\calA^\star$ (bad, $|B|=n-m$).
Let $k=\lceil(n+1)(1-\alpha)\rceil$ and let $\hat q$ be the $k$-th order statistic of all $n$ calibration scores.

\textbf{Step 1 (Worst-case bad scores).}
The adversary controls the values of the $n-m$ bad scores.
To minimize $\hat q$ (and hence minimize coverage), the adversary places all bad scores at $-\infty$.
Then the bottom $n-m$ positions in the sorted order are occupied by bad scores, so
\[
\hat q \ge R^G_{k-(n-m)},
\]
where $R^G_1\le\cdots\le R^G_m$ are the good score order statistics.
(If $k-(n-m)\le 0$, then $\hat q\ge -\infty$, which is trivially true; if $k-(n-m)>m$, then $\hat q\ge\infty$, and coverage is 1.)

\textbf{Step 2 (Exchangeability of good scores).}
The $m+1$ values $\{R^G_1,\ldots,R^G_m,R_i^{(a^\star)}\}$ are exchangeable by Assumption~\ref{ass:selective_exchange}.
By the fundamental property of exchangeable sequences (see \citet[Lemma 1]{shafer2008tutorial} or \citet[Proposition 1]{vovk2005algorithmic}), the test score $R_i^{(a^\star)}$ is equally likely to occupy any of the $m+1$ ranks among these values.
Therefore, for any integer $1\le r\le m$,
\[
\Pp\big(R_i^{(a^\star)} \le R^G_r\big) \ge \frac{r}{m+1}.
\]

\textbf{Step 3 (Combining).}
Setting $r = k-(n-m)$ (assuming $1\le r\le m$; the bound is trivial otherwise):
\[
\Pp(R_i^{(a^\star)}\le\hat q) \ge \Pp(R_i^{(a^\star)}\le R^G_r) \ge \frac{k-(n-m)}{m+1}.
\]

\textbf{Step 4 (Algebra).}
Since $k=\lceil(n+1)(1-\alpha)\rceil\ge (n+1)(1-\alpha)$:
\begin{align*}
\frac{k-(n-m)}{m+1}
&\ge \frac{(n+1)(1-\alpha)-(n-m)}{m+1}\\
&= \frac{m+1-(n+1)\alpha}{m+1}\\
&= 1 - \frac{(n+1)\alpha}{m+1}.
\end{align*}
The coverage gap beyond $1-\alpha$ is:
\begin{align*}
1-\frac{(n+1)\alpha}{m+1} - (1-\alpha)
&= \alpha - \frac{(n+1)\alpha}{m+1}\\
&= \alpha\left(1-\frac{n+1}{m+1}\right)\\
&= -\alpha\cdot\frac{n-m}{m+1}.
\end{align*}
Since $\alpha\le 1$, we bound $\alpha\cdot\frac{n-m}{m+1}\le \frac{n-m}{m+1}$, giving
\[
\Pp(R_i^{(a^\star)}\le\hat q) \ge 1 - \alpha - \frac{n-m}{m+1}.
\]
Substituting $n-m=\delta n$ and $m=(1-\delta)n$:
\[
\frac{n-m}{m+1} = \frac{\delta n}{(1-\delta)n+1} = g(\delta,n). \qedhere
\]
\end{proof}

\textbf{Tightness.}
The bound is tight in the following sense: for any $\delta,n,\alpha$, there exist score distributions (with the bad scores placed adversarially at $-\infty$) such that the coverage equals exactly $\frac{k-(n-m)}{m+1}$, which matches the lower bound up to the ceiling in $k$.

\section{Proof of Corollary~\ref{cor:corrected}}\label{app:corollary}

\begin{proof}
By Theorem~\ref{thm:delta} applied with nominal level $\alpha'$, coverage is at least
\[
\Pp(\ldots)\ge 1-\alpha'-g(\delta,n).
\]
Since $\delta\le \hat\delta$ and $g(\delta,n)$ is nondecreasing in $\delta$, we have $g(\delta,n)\le g(\hat\delta,n)$, hence
\begin{align*}
\Pp(\ldots)
&\ge 1-\alpha'-g(\hat\delta,n)\\
&= 1-(\alpha-g(\hat\delta,n))-g(\hat\delta,n)\\
&= 1-\alpha.
\end{align*}
\end{proof}

\section{Proof of Propositions~\ref{prop:superset} and~\ref{prop:fp}}\label{app:props}

\begin{proof}[Proof of Proposition~\ref{prop:superset}]
Fix $a\in\calA$ and a true descendant $i\in\desc(a)$.
{Since Algorithm~\ref{alg:descendant} sets $\widehat{\desc}(a)=\Cand(a)$, it suffices to show} $i\in\Cand(a) = \calS_a \cap \bigcap_{b\in\calU(a)}\calS_b$.

By Assumption~\ref{ass:mono}, $\Pp(i\in\calS_a)\ge 1-\epsilon_{\mathrm{fn}}$, hence $\Pp(i\notin\calS_a)\le \epsilon_{\mathrm{fn}}$.

For each $b\in\calU(a)$: if $b$ is a true ancestor of $a$ (i.e., $a\in\desc(b)$), then $\desc(a)\subseteq\desc(b)$, so $i\in\desc(b)$.
By Assumption~\ref{ass:mono}, $\Pp(i\in\calS_b)\ge 1-\epsilon_{\mathrm{fn}}$, hence $\Pp(i\notin\calS_b)\le \epsilon_{\mathrm{fn}}$.

The event $\{i\notin\Cand(a)\}$ requires either $i\notin\calS_a$ or $\exists\, b\in\calU(a)$ with $i\notin\calS_b$.
By a union bound:
\[
\Pp(i\notin\Cand(a)) \le (|\calU(a)|+1)\cdot\epsilon_{\mathrm{fn}}.
\]
A further union bound over all $|\desc(a)|$ true descendants gives the result.
Let $E_a=\{\exists\, i\in\desc(a): i\notin\Cand(a)\}$.
\[
\Pp(E_a)
\le |\desc(a)|\cdot(|\calU(a)|+1)\cdot \epsilon_{\mathrm{fn}}.
\]
Hence $\desc(a)\subseteq{\widehat{\desc}(a)}$ with probability at least $1-|\desc(a)|\cdot(|\calU(a)|+1)\cdot\epsilon_{\mathrm{fn}}$.
\end{proof}

\begin{proof}[Proof of Proposition~\ref{prop:fp}]
Fix $a\in\calA$ and $i\notin\desc(a)$.
{First, because $i\notin\desc(a)$, Assumption~\ref{ass:mono} gives $\Pp(i\in\calS_a)\le\epsilon_{\mathrm{fp}}$. Since Algorithm~\ref{alg:descendant} can include $i$ in $\widehat{\desc}(a)$ only if $i\in\calS_a$, this already gives $\Pp(i\in\widehat{\desc}(a))\le\epsilon_{\mathrm{fp}}$.}
By Assumption~\ref{ass:diversity}, there exists $b\in\calA$ such that $a\in\desc(b)$ but $i\notin\desc(b)$ and
\[
\Pp(a\in\calS_b \text{ and } i\notin\calS_b) \ge 1-\epsilon_{\mathrm{cx}}.
\]
On this event, $b$ is identified as upstream of $a$ (since $a\in\calS_b$), but $i\notin\calS_b$.
Thus $i\notin \bigcap_{b'\in\calU(a)}\calS_{b'}$ and therefore $i\notin\Cand(a)$.
Hence $\Pp(i\in\Cand(a))\le \epsilon_{\mathrm{cx}}$.
{Since $\widehat{\desc}(a)=\Cand(a)$ in Algorithm~\ref{alg:descendant}, combining the direct false-positive bound with the upstream-pruning bound gives $\Pp(i\in\widehat{\desc}(a))\le\epsilon_{\mathrm{sel}}=\min\{\epsilon_{\mathrm{fp}},\epsilon_{\mathrm{cx}}\}$.}
Since $\widehat Z_{a,i}=\1\{i\in\widehat{\desc}(a)\}$,
\begin{align*}
\Pp(\widehat Z_{a,i}=1 \mid Z_{a,i}=0)
&= \Pp(i\in\widehat{\desc}(a) \mid i\notin\desc(a))\\
&\le {\epsilon_{\mathrm{sel}}}.
\end{align*}
\end{proof}

{
\section{FPR/FNR Convention for Algorithm~\ref{alg:descendant}}
\label{app:fpr_fnr}

We restate the recovery consequences of Algorithm~\ref{alg:descendant} using the descendant-positive convention
\[
Z_{a,i}=\1\{i\in\desc(a)\},
\qquad
\widehat Z_{a,i}=\1\{i\in\widehat{\desc}(a)\}.
\]
Thus
\[
\FNR_{a,i}=\Pp(\widehat Z_{a,i}=0\mid Z_{a,i}=1)
\]
is the probability that {a calibration intervention that truly affects target $i$} is selected as {safe}; this is the error that contaminates selective conformal calibration. Conversely,
\[
\FPR_{a,i}=\Pp(\widehat Z_{a,i}=1\mid Z_{a,i}=0)
\]
is the probability of discarding a truly safe calibration intervention, which affects efficiency but not validity.

\begin{proposition}[FNR bound for Algorithm~\ref{alg:descendant}]
\label{prop:fnr_restatement}
Under Assumption~\ref{ass:mono}, fix $a\in\calA$ and suppose that $\calU(a)\neq\emptyset$ and every identified upstream intervention is a true ancestor of $a$, i.e.,
\[
b\in\calU(a)\quad\Longrightarrow\quad a\in\desc(b).
\]
Then for any fixed $i\in\desc(a)$,
\[
\Pp(i\notin{\widehat{\desc}(a)})
\le
\big(|\calU(a)|+1\big)\epsilon_{\mathrm{fn}},
\]
and hence under the descendant-positive convention,
\[
\FNR_{a,i}
\le
\big(|\calU(a)|+1\big)\epsilon_{\mathrm{fn}}.
\]
A further union bound over $i\in\desc(a)$ gives
\[
\Pp\big(\desc(a)\subseteq{\widehat{\desc}(a)}\big)
\ge
1-|\desc(a)|\big(|\calU(a)|+1\big)\epsilon_{\mathrm{fn}}.
\]
\end{proposition}

\begin{proof}
Fix $i\in\desc(a)$.  In order for $i$ to belong to
\[
\Cand(a)=\calS_a\cap\bigcap_{b\in\calU(a)}\calS_b,
\]
we need $i\in\calS_a$ and $i\in\calS_b$ for every $b\in\calU(a)$.
By Assumption~\ref{ass:mono},
\[
\Pp(i\notin\calS_a)\le\epsilon_{\mathrm{fn}}.
\]
If $b\in\calU(a)$ is a true ancestor of $a$, then $\desc(a)\subseteq\desc(b)$, hence $i\in\desc(b)$ and again
\[
\Pp(i\notin\calS_b)\le\epsilon_{\mathrm{fn}}.
\]
A union bound over $\calS_a$ and the $|\calU(a)|$ upstream sets gives the per-pair bound
\[
\Pp(i\notin\Cand(a))
\le
\big(|\calU(a)|+1\big)\epsilon_{\mathrm{fn}}.
\]
Since $\widehat{\desc}(a)=\Cand(a)$ in Algorithm~\ref{alg:descendant},
\[
\Pp(i\notin{\widehat{\desc}(a)})
=
\Pp(i\notin\Cand(a))
\le
\big(|\calU(a)|+1\big)\epsilon_{\mathrm{fn}},
\]
and
\[
\FNR_{a,i}
=
\Pp(\widehat Z_{a,i}=0\mid Z_{a,i}=1)
=
\Pp(i\notin\Cand(a)\mid i\in\desc(a))
\le
\big(|\calU(a)|+1\big)\epsilon_{\mathrm{fn}}.
\]
Taking a further union bound over all $i\in\desc(a)$ yields
\[
\Pp\big(\exists\, i\in\desc(a):i\notin\Cand(a)\big)
\le
\sum_{i\in\desc(a)}\Pp(i\notin\Cand(a))
\le
|\desc(a)|\big(|\calU(a)|+1\big)\epsilon_{\mathrm{fn}},
\]
which is equivalent to the displayed lower bound for $\Pp(\desc(a)\subseteq{\widehat{\desc}(a)})$.
\end{proof}

\begin{proposition}[FPR bound for Algorithm~\ref{alg:descendant}]
\label{prop:fpr_restatement}
Under Assumptions~\ref{ass:mono}--\ref{ass:diversity}, for every $a\in\calA$ and every non-descendant $i\notin\desc(a)$,
\[
\Pp(\widehat Z_{a,i}=1\mid Z_{a,i}=0)
\le
{\epsilon_{\mathrm{sel}}}.
\]
Equivalently, under the descendant-positive convention,
\[
\FPR_{a,i}\le{\epsilon_{\mathrm{sel}}},
\qquad
{\epsilon_{\mathrm{sel}}=\min\{\epsilon_{\mathrm{fp}},\epsilon_{\mathrm{cx}}\}.}
\]
\end{proposition}

\begin{proof}
Fix $i\notin\desc(a)$.
{Since $i\notin\desc(a)$, Assumption~\ref{ass:mono} gives $\Pp(i\in\calS_a)\le\epsilon_{\mathrm{fp}}$, and Algorithm~\ref{alg:descendant} can include $i$ only if $i\in\calS_a$. Hence $\Pp(i\in\widehat{\desc}(a))\le\epsilon_{\mathrm{fp}}$.}
By Assumption~\ref{ass:diversity}, there exists $b\in\calA$ such that $a\in\desc(b)$, $i\notin\desc(b)$, and
\[
\Pp(a\in\calS_b\ \text{and}\ i\notin\calS_b)
\ge
1-\epsilon_{\mathrm{cx}}.
\]
On this event, $b$ is included in the upstream set $\calU(a)$ while $i$ is absent from $\calS_b$.
Therefore $i$ is removed by the intersection step and $i\notin\Cand(a)$.
Since $\widehat{\desc}(a)=\Cand(a)$ in Algorithm~\ref{alg:descendant},
{this gives the second bound $\Pp(i\in\widehat{\desc}(a))\le\epsilon_{\mathrm{cx}}$. Combining the two bounds yields $\Pp(i\in\widehat{\desc}(a))\le\epsilon_{\mathrm{sel}}$.}
\[
\Pp(\widehat Z_{a,i}=1\mid Z_{a,i}=0)
=
\Pp(i\in\widehat{\desc}(a)\mid i\notin\desc(a))
\le
{\epsilon_{\mathrm{sel}}}.
\]
\end{proof}

\begin{corollary}[Expected contamination under corrected labels]
\label{cor:contamination_restatement}
Fix target $i$ and calibration {split} $\calA_{\mathrm{cal}}$.
Let $\pi_0$ be the fraction of truly {safe} interventions in $\calA_{\mathrm{cal}}${, meaning interventions that do not affect target $i$}:
\[
\pi_0
=
\frac{|\{a\in\calA_{\mathrm{cal}}:Z_{a,i}=0\}|}
{|\calA_{\mathrm{cal}}|}.
\]
Define the worst-case bounds
\[
\FNR_{\max}
=
\max_{a\in\calA_{\mathrm{cal}}}
\big(|\calU(a)|+1\big)\epsilon_{\mathrm{fn}}
=
(\bar u+1)\epsilon_{\mathrm{fn}},
\qquad
\FPR_{\max}
=
{\epsilon_{\mathrm{sel}}}.
\]
Here $\bar u=\max_{a\in\calA_{\mathrm{cal}}}|\calU(a)|$.
Under Assumptions~\ref{ass:mono}--\ref{ass:diversity} and the upstream-ancestor condition in Proposition~\ref{prop:fnr_restatement},
\[
\E[\delta]
\le
\frac{(1-\pi_0)\FNR_{\max}}
{\pi_0(1-\FPR_{\max})+(1-\pi_0)\FNR_{\max}}.
\]
Equivalently,
\[
\E[\delta]
\le
\frac{(1-\pi_0)(\bar u+1)\epsilon_{\mathrm{fn}}}
{\pi_0(1-{\epsilon_{\mathrm{sel}}})+(1-\pi_0)(\bar u+1)\epsilon_{\mathrm{fn}}}.
\]
\end{corollary}

\begin{proof}
Under the descendant-positive convention, contamination occurs exactly when $Z_{a,i}=1$ but $\widehat Z_{a,i}=0$.
Let $A$ denote the expected selected mass among {interventions that truly affect target $i$} and $S$ the expected selected mass among {safe} interventions.
By the per-pair FNR bound in Proposition~\ref{prop:fnr_restatement} and the FPR bound in Proposition~\ref{prop:fpr_restatement},
\[
A\le (1-\pi_0)\FNR_{\max},
\qquad
S\ge \pi_0(1-\FPR_{\max}).
\]
Using the contamination expression from Section~\ref{sec:taskdriven}, $\E[\delta]=A/(S+A)$ at the level of expected selected masses.
This expression is increasing in $A$ and decreasing in $S$, so
\[
\E[\delta]
\le
\frac{(1-\pi_0)\FNR_{\max}}
{\pi_0(1-\FPR_{\max})+(1-\pi_0)\FNR_{\max}}.
\]
Equivalently, at the level of expected or population rates, the selected {safe} mass is at least
\[
\pi_0(1-\FPR_{\max})
\]
and the contaminating selected mass is at most $(1-\pi_0)\FNR_{\max}$.
The bounds on $\FNR_{\max}$ and $\FPR_{\max}$ are exactly Propositions~\ref{prop:fnr_restatement} and~\ref{prop:fpr_restatement}.
The factor $\bar u+1$ counts the own affected set $\calS_a$ together with the $|\calU(a)|$ upstream affected sets in the union bound.
\end{proof}

\begin{remark}[Validity versus efficiency]
With these labels, Proposition~\ref{prop:fnr_restatement} is the validity-relevant statement: false negatives under the descendant-positive convention are {interventions that truly affect target $i$ but are incorrectly admitted} into $\widehat{\calA}(i,a^\star)$, and hence they drive $\delta$ and the coverage loss in Theorem~\ref{thm:delta}. Proposition~\ref{prop:fpr_restatement} is efficiency-relevant: false positives discard safe calibration interventions and can widen intervals, but they do not introduce non-exchangeable scores into the selected calibration set.
\end{remark}
}

{
\section{A High-Probability Bound on the Contamination Fraction}
\label{app:prob_delta}

The correction in Corollary~\ref{cor:corrected} requires an upper bound on the realized contamination fraction
\[
\delta
=\frac{|\widehat{\calA}(i,a^\star)\setminus \calA^\star(i)|}
{|\widehat{\calA}(i,a^\star)|}.
\]
Corollary~\ref{cor:contamination} controls this quantity in expectation.  We next give a high-probability version under an explicit weak-dependence condition on the classification errors across calibration interventions.

\begin{assumption}[Fixed selected size and weak dependence]
\label{ass:weak_dep_delta}
Fix a target $i$ and test intervention $a^\star$ with $Z_{a^\star,i}=0$.
Assume that $|\widehat{\calA}(i,a^\star)|=n$ is fixed by design, or condition on an event on which $|\widehat{\calA}(i,a^\star)|=n$ and the error bounds below hold conditionally.
For each $a\in\calA_{\mathrm{cal}}$, define
\[
W_a
=\1\{a\in\widehat{\calA}(i,a^\star),\ Z_{a,i}=1\}.
\]
There is a dependency graph $H_i$ on $\calA_{\mathrm{cal}}$ with maximum degree $\Delta_i$ such that any two subcollections of $\{W_a:a\in\calA_{\mathrm{cal}}\}$ indexed by vertex sets with no edge between them are independent.
\end{assumption}

Assumption~\ref{ass:weak_dep_delta} is automatic with $\Delta_i=0$ if the descendant-classification errors are independent across interventions.  The dependency-graph form allows the upstream-intersection step in Algorithm~\ref{alg:descendant} to induce local dependence through shared differentially affected sets.
When the selected set size is random and no conditioning is imposed, the same argument applies with $n$ replaced by any deterministic lower bound on $|\widehat{\calA}(i,a^\star)|$.

\begin{proposition}[High-probability contamination bound]
\label{prop:prob_delta}
Under Assumption~\ref{ass:weak_dep_delta}, let
\[
C_i=\sum_{a\in\calA_{\mathrm{cal}}} W_a,
\qquad
\mu_i=\E[C_i],
\qquad
\sigma_i^2=\mathrm{Var}(C_i).
\]
Then $\delta=C_i/n$ and, for every $t>0$,
\[
\Pp\left(\delta>\frac{\mu_i}{n}+t\right)
\le
\frac{\sigma_i^2}{n^2t^2}.
\]
Moreover, if $q_a=\Pp(W_a=1)$, then
\[
\sigma_i^2
\le
\sum_{a\in\calA_{\mathrm{cal}}}q_a(1-q_a)
+2\sum_{\{a,b\}\in E(H_i)}\sqrt{q_a q_b}
\le
(\Delta_i+1)\mu_i.
\]
Consequently, for any $\eta\in(0,1)$, with probability at least $1-\eta$ over the realization of $\widehat Z$,
\[
\delta
\le
\widehat\delta_{\eta}
\equiv
\frac{\mu_i}{n}
+\frac{\sigma_i}{n\sqrt{\eta}}
\le
\frac{\mu_i}{n}
+\frac{\sqrt{(\Delta_i+1)\mu_i}}{n\sqrt{\eta}}.
\]
Under the descendant-positive convention, if
\[
\FNR_a
=\Pp(\widehat Z_{a,i}=0\mid Z_{a,i}=1),
\]
then
\[
\mu_i
=\sum_{a:Z_{a,i}=1}\Pp(\widehat Z_{a,i}=0)
\le
\sum_{a:Z_{a,i}=1}\FNR_a.
\]
\end{proposition}

\begin{proof}
By definition, $W_a=1$ exactly when calibration intervention $a$ is selected but {truly affects target $i$}. Hence
\[
C_i=\sum_{a\in\calA_{\mathrm{cal}}}W_a
=|\widehat{\calA}(i,a^\star)\setminus \calA^\star(i)|.
\]
Conditioning on $|\widehat{\calA}(i,a^\star)|=n$ gives $\delta=C_i/n$.

Chebyshev's inequality gives, for every $t>0$,
\[
\Pp\left(\delta>\frac{\mu_i}{n}+t\right)
=
\Pp(C_i-\mu_i>nt)
\le
\Pp(|C_i-\mu_i|>nt)
\le
\frac{\sigma_i^2}{n^2t^2}.
\]

It remains to bound $\sigma_i^2$.  Since $W_a$ is Bernoulli,
\[
\mathrm{Var}(W_a)=q_a(1-q_a).
\]
By the dependency-graph assumption, $\mathrm{Cov}(W_a,W_b)=0$ whenever $\{a,b\}\notin E(H_i)$.
For edges of $H_i$, Cauchy--Schwarz gives
\[
|\mathrm{Cov}(W_a,W_b)|
\le
\sqrt{\mathrm{Var}(W_a)\mathrm{Var}(W_b)}
\le
\sqrt{q_a q_b}.
\]
Therefore
\[
\mathrm{Var}(C_i)
\le
\sum_a q_a(1-q_a)
+2\sum_{\{a,b\}\in E(H_i)}\sqrt{q_a q_b}.
\]
Using $2\sqrt{q_a q_b}\le q_a+q_b$ and the fact that each vertex has degree at most $\Delta_i$,
\[
2\sum_{\{a,b\}\in E(H_i)}\sqrt{q_a q_b}
\le
\sum_{\{a,b\}\in E(H_i)}(q_a+q_b)
\le
\Delta_i\sum_a q_a
=\Delta_i\mu_i.
\]
Since $\sum_a q_a(1-q_a)\le\sum_a q_a=\mu_i$, we obtain
$\sigma_i^2\le(\Delta_i+1)\mu_i$.
The displayed $1-\eta$ bound follows by taking $t=\sigma_i/(n\sqrt{\eta})$.
Finally, if $Z_{a,i}=0$ then $W_a=0$ deterministically; if $Z_{a,i}=1$ then
$W_a=1$ implies $\widehat Z_{a,i}=0$, so
$\Pp(W_a=1)\le\Pp(\widehat Z_{a,i}=0\mid Z_{a,i}=1)=\FNR_a$.
\end{proof}

\begin{corollary}[High-probability corrected selective conformal]
\label{cor:hp_corrected}
Fix $\eta\in(0,1)$ and suppose the assumptions of Proposition~\ref{prop:prob_delta} hold.
Let $\widehat\delta_{\eta}$ be any deterministic quantity satisfying
\[
\Pp(\delta\le\widehat\delta_{\eta})\ge 1-\eta.
\]
For example, one may use the bound in Proposition~\ref{prop:prob_delta}, with $\mu_i$ further upper bounded by the FNR bounds from Proposition~\ref{prop:fnr_restatement}.
Then, with probability at least $1-\eta$ over the realized descendant estimator $\widehat Z$, the split conformal interval computed at nominal miscoverage
\[
\alpha'=\alpha-g(\widehat\delta_{\eta},n)
\]
has conditional coverage at least $1-\alpha$ for $R_i^{(a^\star)}$.
If $\alpha'\le0$, the corrected interval is taken to be $(-\infty,\infty)$ and the claim is trivial.
\end{corollary}

\begin{proof}
On the event $\{\delta\le\widehat\delta_{\eta}\}$, Corollary~\ref{cor:corrected} applies with upper bound $\widehat\delta_{\eta}$ and gives conditional coverage at least $1-\alpha$.
Proposition~\ref{prop:prob_delta} gives this event probability at least $1-\eta$.
\end{proof}

\begin{remark}[When is the bound informative?]
The Chebyshev correction is useful when the additive term $\sigma_i/(n\sqrt{\eta})$ is small compared with the target slack in $\delta$.
For a relative bound of the form $\delta\le(1+c)\E[\delta]$, Chebyshev gives failure probability at most
\[
\frac{\mathrm{Var}(\delta)}{c^2{\E[\delta]}^2}
=
\frac{\sigma_i^2}{c^2\mu_i^2}
\le
\frac{\Delta_i+1}{c^2\mu_i}.
\]
Thus relative concentration requires $\mu_i\gg\Delta_i$, i.e., the expected number of contaminating selected interventions must dominate the local dependence degree.
For coverage correction, however, the relevant requirement is often absolute rather than relative:
\[
\frac{\sqrt{(\Delta_i+1)\mu_i}}{n\sqrt{\eta}}\ll 1.
\]
This can hold even when $\E[\delta]=\mu_i/n$ is small, provided the selected calibration set is large and the dependency graph is sparse.  If the upstream-intersection errors are highly coupled across many interventions, so that $\Delta_i$ is of the same order as $|\calA_{\mathrm{cal}}|$, this Chebyshev bound becomes conservative; sharper concentration would require a more detailed model of the dependence induced by Algorithm~\ref{alg:descendant}.
\end{remark}
}

\section{Algorithm 2: Local ICP for Distance Estimation}\label{app:alg_icp}

\begin{algorithm}[h]
\caption{Local ICP for Distance Estimation}
\label{alg:local_icp}
\begin{algorithmic}[1]
\footnotesize
\REQUIRE Target gene $i$, interventions $\calA$, data $\{Y^{(a)}\}_{a\in\calA}$, maximum depth $D$
\ENSURE Estimated distance $\widehat d(a,i)$ for each $a\in\calA$
\STATE $\calP_0 \gets$ candidate parent pool via correlation screening or prior network
  \STATE $\widehat{\mathrm{Pa}}(i) \gets \mathrm{ICP}(i, \calP_0, \calA)$ \COMMENT{Invariant parent set}
  \STATE $\calB_0 \gets \{i\}$
  \FOR{$t = 1,\ldots,D$}
    \STATE $\calB_t \gets \calB_{t-1}$
    \STATE \COMMENT{$\calB_{-1}=\emptyset$}
    \FOR{each $g\in\calB_{t-1}\setminus\calB_{t-2}$}
      \STATE $\calP_0^{(g)} \gets$ candidate parent pool for $g$
      \STATE $\widehat{\mathrm{Pa}}(g) \gets \mathrm{ICP}(g, \calP_0^{(g)}, \calA)$
      \STATE $\calB_t \gets \calB_t \cup \widehat{\mathrm{Pa}}(g)$
    \ENDFOR
  \ENDFOR
\FOR{each $a\in\calA$}
  \STATE $\widehat d(a,i) \gets \min\{t: a\in\calB_t\}$ \COMMENT{$\infty$ if $a\notin\calB_D$}
\ENDFOR
\RETURN $\{\widehat d(a,i)\}_{a\in\calA}$
\end{algorithmic}
\end{algorithm}

The ICP subroutine works as follows: for a target gene $g$ and candidate parent pool $\calP_0^{(g)}$, test subsets $S\subseteq\calP_0^{(g)}$ for invariance of the regression residual $Y_g - \hat\beta_S^\top X_S$ across environments (interventions), and return a minimal invariant set as $\widehat{\mathrm{Pa}}(g)$.
The frontier expansion traces back from $i$ through estimated parents, yielding a coarse path-length estimate without learning the full graph.
{For runtime accounting, $\calP_0^{(g)}$ denotes the screened candidate parent pool used when ICP is called for gene $g$; the initial pool $\calP_0$ in Algorithm~\ref{alg:local_icp} is $\calP_0^{(i)}$. Let $K_0=\max_g|\calP_0^{(g)}|$, let $k_{\mathrm{ICP}}$ be the maximum parent-set size searched by the ICP subroutine, let $D$ be the maximum backward search depth, and let $M_D(i)=|\cup_{t=0}^D\calB_t|$ be the number of genes visited during the local expansion from target $i$. Each ICP call examines $O(K_0^{k_{\mathrm{ICP}}})$ candidate subsets, so the subset-search cost is $O(M_D(i)K_0^{k_{\mathrm{ICP}}})$, up to the cost of the individual invariance tests. In the single-gene perturbation setting with the expansion restricted to intervention-indexed genes, the shorthand bound is $O(|\calA|D K_0^{k_{\mathrm{ICP}}})$. Thus correlation screening makes the runtime depend on the screened pool size $K_0$ and the local visited set, rather than directly on the ambient number of genes $p$.}

\textbf{Using distance for weighted calibration.}
Given distance estimates $\widehat d(a,i)$, one can define weighted conformal calibration:
\[
w(a) = K\!\left(\frac{\widehat d(a,i)}{h}\right),
\]
where $K$ is a kernel function and $h$ is a bandwidth parameter.
This provides a smooth interpolation between selective calibration (hard thresholding on $\widehat Z$) and pooled calibration, potentially improving efficiency when the binary classification is uncertain.

{
\section{Consistency of Algorithm~\ref{alg:local_icp}}
\label{app:local_icp_consistency}

We give a recovery guarantee for the Local ICP distance estimator in Algorithm~\ref{alg:local_icp}.  Recall that the algorithm starts from $\calB_0=\{i\}$ and recursively expands backward through estimated invariant parent sets.  It returns
\[
\widehat d(a,i)=\min\{t:\ a\in\calB_t\},
\]
with $\widehat d(a,i)=\infty$ if $a\notin\calB_D$.

For a node $g$, let $\mathrm{Pa}(g)$ denote its true parent set in the underlying DAG $G$, and let $d(a,i)$ be the length of the shortest directed path from $a$ to $i$, with $d(a,i)=\infty$ if no such path exists.  Define
\[
\mathrm{Anc}^{\le D}(i)
=\{g:\ d(g,i)\le D\},
\]
where $i\in\mathrm{Anc}^{\le D}(i)$ with $d(i,i)=0$.

\begin{proposition}[Local ICP distance recovery]
\label{prop:local_icp_recovery}
Assume:
\begin{enumerate}
\item the data are generated from a causally sufficient DAG $G$;
\item interventional faithfulness holds, so invariant causal prediction identifies the true invariant parent set in the population;
\item for every gene $g$ that can be visited by Algorithm~\ref{alg:local_icp}, the candidate parent pool covers the true parents,
\[
\calP_0^{(g)}\supseteq \mathrm{Pa}(g);
\]
\item the ICP subroutine satisfies the finite-sample recovery guarantee
\[
\Pp\big(\widehat{\mathrm{Pa}}(g)=\mathrm{Pa}(g)\big)\ge 1-\beta_g
\]
for every $g\in\mathrm{Anc}^{\le D}(i)$.
\end{enumerate}
Then, with probability at least
\[
1-\sum_{g\in\mathrm{Anc}^{\le D}(i)}\beta_g,
\]
Algorithm~\ref{alg:local_icp} recovers the exact truncated ancestor balls:
\[
\calB_t=\{g:\ d(g,i)\le t\}
\qquad\text{for all }t=0,\ldots,D.
\]
Consequently, for every $a\in\calA$ with $d(a,i)\le D$,
\[
\Pp\big(\widehat d(a,i)=d(a,i)\big)
\ge
1-\sum_{g\in\mathrm{Anc}^{\le D}(i)}\beta_g,
\]
and for every $a\in\calA$ with $d(a,i)>D$,
\[
\Pp\big(\widehat d(a,i)=\infty\big)
\ge
1-\sum_{g\in\mathrm{Anc}^{\le D}(i)}\beta_g.
\]
\end{proposition}

\begin{proof}
Let
\[
E_i=\bigcap_{g\in\mathrm{Anc}^{\le D}(i)}
\{\widehat{\mathrm{Pa}}(g)=\mathrm{Pa}(g)\}.
\]
By a union bound,
\[
\Pp(E_i)
\ge
1-\sum_{g\in\mathrm{Anc}^{\le D}(i)}\beta_g.
\]
It is enough to prove the deterministic claim on $E_i$.

We proceed by induction on $t$.
For $t=0$, Algorithm~\ref{alg:local_icp} initializes $\calB_0=\{i\}$, which equals $\{g:d(g,i)\le0\}$.
Assume for some $t-1<D$ that
\[
\calB_{t-1}=\{g:\ d(g,i)\le t-1\}.
\]
The new frontier is $\calB_{t-1}\setminus\calB_{t-2}$, which by the induction hypothesis is exactly the set of nodes at distance $t-1$ from $i$.
For each such node $g$, on the event $E_i$ the ICP call returns
\[
\widehat{\mathrm{Pa}}(g)=\mathrm{Pa}(g).
\]
Therefore the update
\[
\calB_t=\calB_{t-1}\cup
\bigcup_{g\in\calB_{t-1}\setminus\calB_{t-2}}\widehat{\mathrm{Pa}}(g)
\]
equals
\[
\{g:\ d(g,i)\le t-1\}
\cup
\bigcup_{h:\ d(h,i)=t-1}\mathrm{Pa}(h).
\]
Every parent of a node at distance $t-1$ has a directed path of length $t$ to $i$, so the right-hand side is contained in $\{g:d(g,i)\le t\}$.
Conversely, if $u$ satisfies $d(u,i)=t$, then there exists a shortest directed path
\[
u\to h\to\cdots\to i
\]
where $d(h,i)=t-1$ and $u\in\mathrm{Pa}(h)$.  Since $h$ is in the frontier at depth $t-1$, the algorithm adds $u$ at step $t$.
Thus $\calB_t=\{g:d(g,i)\le t\}$.

The induction proves the exact recovery of all truncated ancestor balls through depth $D$ on $E_i$.
If $d(a,i)\le D$, then $a$ first appears in $\calB_{d(a,i)}$ and does not appear in any earlier $\calB_t$, so $\widehat d(a,i)=d(a,i)$.
If $d(a,i)>D$, then $a\notin\calB_D$, so the algorithm returns $\widehat d(a,i)=\infty$.
Combining these deterministic implications with the lower bound on $\Pp(E_i)$ proves the result.
\end{proof}

\begin{corollary}[Finite-sample consistency]
\label{cor:local_icp_consistency}
Suppose the assumptions of Proposition~\ref{prop:local_icp_recovery} hold and that, for each fixed $g\in\mathrm{Anc}^{\le D}(i)$, the ICP error probability satisfies
\[
\beta_g=\beta_g(n_{\min})\to0
\qquad\text{as}\qquad
n_{\min}=\min_{a\in\calA} n_a\to\infty,
\]
where $n_a$ is the sample size under intervention $a$.
If $|\mathrm{Anc}^{\le D}(i)|<\infty$, then for every $a\in\calA$,
\[
\widehat d(a,i)\ \xrightarrow{p}\
\begin{cases}
d(a,i), & d(a,i)\le D,\\
\infty, & d(a,i)>D.
\end{cases}
\]
\end{corollary}

\begin{proof}
By Proposition~\ref{prop:local_icp_recovery}, the probability of any truncated-distance recovery error is at most
\[
\sum_{g\in\mathrm{Anc}^{\le D}(i)}\beta_g(n_{\min}).
\]
The ancestor set is finite and each summand converges to zero, so the sum converges to zero.
\end{proof}

\begin{remark}[Implication for weighted conformal calibration]
Appendix~\ref{app:alg_icp} proposes kernel weights
\[
w(a)=K\!\left(\frac{\widehat d(a,i)}{h}\right)
\]
for distance-weighted calibration.  Proposition~\ref{prop:local_icp_recovery} shows that, under the stated ICP recovery conditions, these estimated weights equal the oracle weights $K(d(a,i)/h)$ with probability tending to one as the ICP subroutine becomes consistent.  Therefore any exchangeability or weighted-exchangeability guarantee proved for the oracle distance weights is inherited asymptotically by the Local ICP weights.  This does not by itself give a new finite-sample weighted-conformal theorem; it reduces the additional error from estimating $d(a,i)$ to the ICP recovery probability in Proposition~\ref{prop:local_icp_recovery}.
\end{remark}
}

\section{Additional Experimental Results}\label{app:extra}

\subsection{{Procedural definition of table methods}}\label{app:method_definitions}

{For a fixed test intervention--target pair $(a^\star,i)$, let $r_a=R_i^{(a)}$ be the calibration score for intervention $a\in\calA_{\mathrm{cal}}$ and let $r^\star=R_i^{(a^\star)}$ be the test score. For any selected calibration set $S\subseteq\calA_{\mathrm{cal}}$, split conformal uses the order statistic}
\[
{
\widehat q_\alpha(S)
=\text{the }\left\lceil(|S|+1)(1-\alpha)\right\rceil\text{-th smallest value in }
\{r_a:a\in S\}\cup\{\infty\}.}
\]
{Coverage for the score-based experiments is the indicator $\1\{r^\star\le \widehat q_\alpha(S)\}$ and width is $2\widehat q_\alpha(S)$.}
{Here $S$ is only a local shorthand for a selected intervention set: $S_{\mathrm{oracle}}(i)=\calA^\star(i)$ and $S_{\mathrm{est}}(i)=\widehat{\calA}(i,a^\star)$ in the notation of the main text.}

{\textbf{Oracle.}
The oracle selector uses the truly safe calibration set $S_{\mathrm{oracle}}(i)=\{a\in\calA_{\mathrm{cal}}:Z_{a,i}=0\}$. In synthetic experiments $Z$ is known from the simulated DAG. In the real-data experiment this row is only a proxy oracle, with $Z$ defined by the top-$10\%$ absolute log-fold-change rule.}

{\textbf{Estimated.}
The estimated selector uses Algorithm~\ref{alg:descendant} to form $\widehat Z$ and selects $S_{\mathrm{est}}(i)=\{a\in\calA_{\mathrm{cal}}:\widehat Z_{a,i}=0\}$. If the selected set is too small to yield a finite split-conformal quantile, the implementation falls back to pooled calibration.}

{\textbf{Pooled.}
The pooled baseline ignores causal strata and uses all calibration interventions, $S_{\mathrm{pooled}}=\calA_{\mathrm{cal}}$. It is included as a coarse baseline rather than as the target exchangeable set for fixed safe test pairs.}

{\textbf{Corrected.}
The corrected method uses $S_{\mathrm{est}}$ but replaces $\alpha$ by $\alpha'=\alpha-g(\widehat\delta,|S_{\mathrm{est}}|)$, where {$\widehat\delta$ is treated as an externally supplied upper bound; in the experiments we plug in realized synthetic contamination or proxy real-data contamination}. The formal correction returns an infinite interval when $\alpha'\le0$; the experiments report finite-width behavior under the clipped implementation described in the main text.}

{\textbf{Control-coexpression weighted CP.}
For perturbation $a$, let $g(a)$ be its target gene and let $x_a=(\log(1+X_{c,g(a)}))_{c\in\mathrm{control}}$ be the vector of control-cell expression values for that gene. We compute a correlation distance between $x_a$ and $x_{a^\star}$ and assign Gaussian kernel weights to calibration interventions, with bandwidth set by the tenth nearest calibration perturbation. This baseline is leakage-free because it uses only control-cell expression of the perturbed target genes, not the held-out perturbation's response vector. The reported $n_{\mathrm{cal}}$ is the Kish effective sample size of the weights.}

{\textbf{Full-graph proxy.}
This baseline constructs a proxy gene graph from correlations among perturbation-level LFC profiles, thresholds the top $1\%$ of absolute correlations to form a skeleton, orients edges by perturbation-level variance, and uses graph descendants as affected labels before applying the same split-conformal procedure as Estimated.}

{\textbf{Observational-only selection.}
This heuristic also uses the correlation matrix but does not run Algorithm~\ref{alg:descendant}. For each perturbed target, it labels the most correlated genes as affected, with the number of selected genes matched to the average descendant-set size produced by Algorithm~\ref{alg:descendant}, and then applies split conformal to the remaining calibration interventions.}

\begin{figure}[h]
\centering
\includegraphics[width=0.6\linewidth]{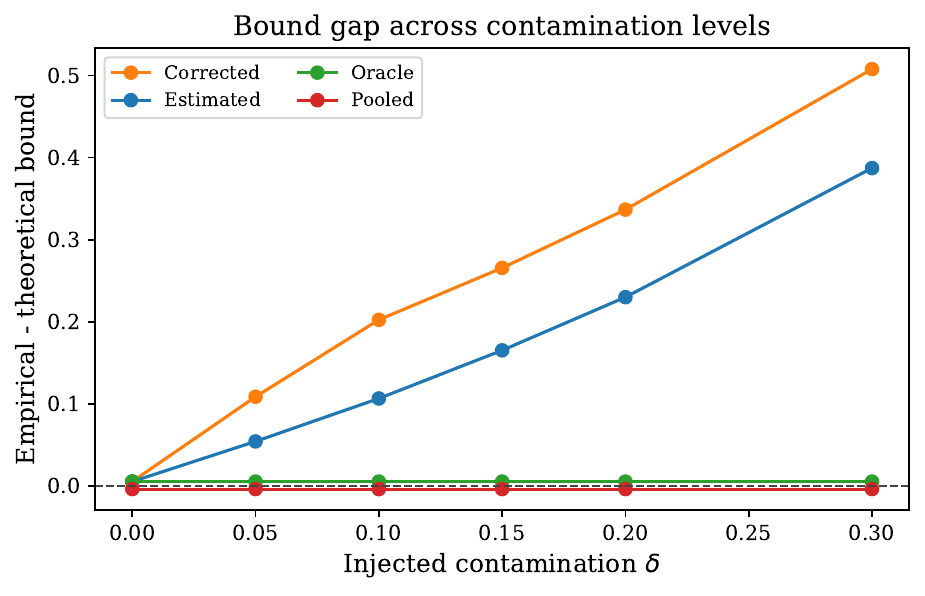}
\caption{Gap between empirical coverage and the theoretical lower bound from Theorem~\ref{thm:delta}. All values are non-negative for the selective methods (Oracle, Estimated, Corrected), confirming the bound is valid. Pooled shows a small negative gap ($\approx -0.004$) because it uses all calibration points without selection, so the selective coverage theorem does not apply. The gap grows with $\delta$ because worst-case adversarial contamination does not occur in practice.}
\label{fig:bound_gap}
\end{figure}

\begin{table}[h]
\centering
\caption{Controlled $\delta$ ablation. Coverage and width by injected contamination level. Corrected maintains $\ge 0.95$ coverage at all nonzero contamination levels ($\delta\ge 0.05$) {when supplied with the injected $\delta$} {under the clipped finite-width implementation}. Oracle and Pooled widths (constant at $3.38$ and {$3.29$} respectively) are omitted, {so the displayed width rows are interpreted as the uncorrected--corrected tradeoff rather than a complete width ranking}.}
\label{tab:delta}
\footnotesize
\begin{tabular}{lcccccc}
\toprule
& \multicolumn{6}{c}{$\delta_{\mathrm{inject}}$} \\
\cmidrule(lr){2-7}
Method & $0.00$ & $0.05$ & $0.10$ & $0.15$ & $0.20$ & $0.30$ \\
\midrule
\multicolumn{7}{l}{\emph{Coverage}} \\
Oracle & $\okcov{.905}$ & $\okcov{.905}$ & $\okcov{.905}$ & $\okcov{.905}$ & $\okcov{.905}$ & $\okcov{.905}$ \\
Estimated & $\okcov{.905}$ & $\okcov{.901}$ & $.895$ & $.889$ & $.882$ & $.867$ \\
Pooled & $.896$ & $.896$ & $.896$ & $.896$ & $.896$ & $.896$ \\
Corrected & $\okcov{.905}$ & $\okcov{.955}$ & $\okcov{.990}$ & $\okcov{.990}$ & $\okcov{.989}$ & $\okcov{.988}$ \\
\midrule
\multicolumn{7}{l}{\emph{Width}} \\
Estimated & $3.38$ & $3.33$ & $3.28$ & $3.22$ & $3.16$ & $3.03$ \\
Corrected & $3.38$ & $4.09$ & $5.52$ & $5.48$ & $5.44$ & $5.35$ \\
\bottomrule
\end{tabular}
\end{table}

\textbf{Feasibility versus calibration-set size.}
{The corrected procedure's $59.8\%$ feasibility on the real-data split (Table~\ref{tab:real}) reflects the small calibration set rather than a methodological limit. Re-running the Replogle K562 evaluation while scaling the calibration-set size shows feasibility climbing to $100\%$ by $n_{\mathrm{cal}}=160$, with the {proxy contamination fraction $\delta$} stable near $0.083$ (Table~\ref{tab:feasibility}, Figure~\ref{fig:feasibility}). {All rows in this scaling study are produced by bootstrap-resampling the selected real calibration residuals to the target $n_{\mathrm{cal}}$, rather than by adding real perturbations,} so their coverage and width should be read as indicative; feasibility, a function of {$g(\delta,n)$ using this proxy value} versus $\alpha$, is robust to this expansion.}

\begin{figure}[h]
\centering
\includegraphics[width=0.6\linewidth]{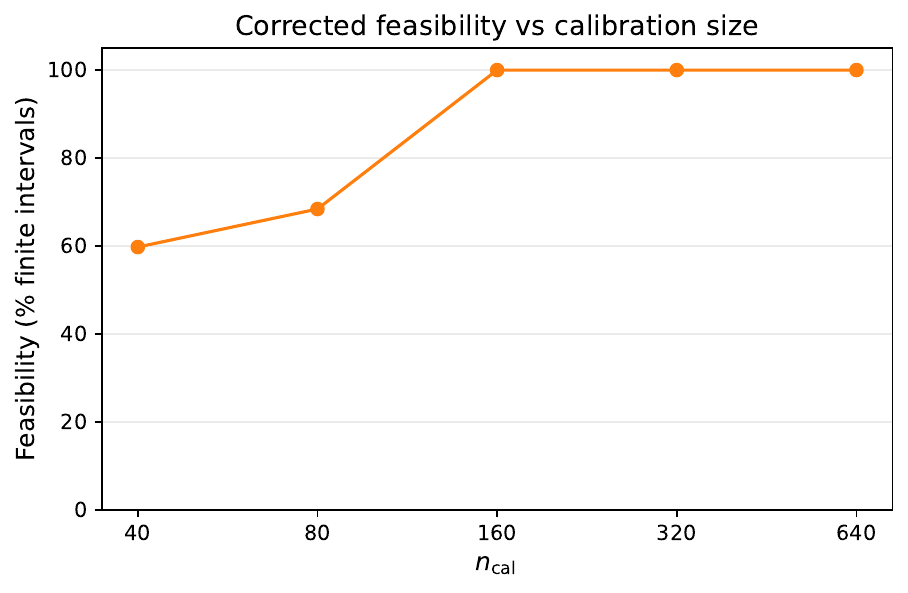}
\caption{{Feasibility (fraction of finite intervals) of the Corrected procedure versus calibration-set size $n_{\mathrm{cal}}$ on the Replogle K562 split. Feasibility rises from $0.598$ at $n_{\mathrm{cal}}=40$ to $1.0$ by $n_{\mathrm{cal}}=160$.}}
\label{fig:feasibility}
\end{figure}

\begin{table}[h]
\centering
\caption{{Corrected procedure versus calibration-set size $n_{\mathrm{cal}}$ on the Replogle K562 split. {Rows use bootstrap-resampled calibration residuals at the displayed target sizes.} Feasibility reaches $100\%$ by $n_{\mathrm{cal}}=160$ while {proxy $\delta$} stays near $0.083$. {Width highlighting ranks only rows that also exceed nominal coverage.}}}
\label{tab:feasibility}
\footnotesize
\begin{tabular}{lcccc}
\toprule
{$n_{\mathrm{cal}}$} & {Coverage} & {Width} & {Feasibility} & {{proxy $\delta$}} \\
\midrule
{$40$} & {$0.890$} & {$0.316$} & {$0.598$} & {$0.083$} \\
{$80$} & {$0.898$} & {$0.338$} & {$0.684$} & {$0.083$} \\
{$160$} & {$\okcov{0.928}$} & {$0.520$} & {$1.000$} & {$0.083$} \\
{$320$} & {$\okcov{0.925}$} & {$\bestw{0.517}$} & {$1.000$} & {$0.083$} \\
{$640$} & {$\okcov{0.926}$} & {$\secondw{0.519}$} & {$1.000$} & {$0.083$} \\
\bottomrule
\end{tabular}
\end{table}

{
\section{Additional Robustness Experiments}\label{app:robustness}

\textbf{Extended contamination range.}
We extend the controlled-$\delta$ ablation of Section~\ref{sec:exp_delta} to injected contamination up to $\delta=0.70$ (same setting: $p=200$, $|\calA|=150$, $100$ seeds). Table~\ref{tab:delta_extended} reports the Corrected procedure. Coverage stays at or above the nominal level across the entire range, and the intervals remain finite (feasibility $\approx 1$) even at $\delta=0.70$: at the synthetic calibration size ($n_{\mathrm{cal}}\approx120$) the corrected level {is clipped at $\alpha'=0.01$ for finite-width reporting}, still admitting a finite conformal quantile. {Without this clipping, rows where {$\alpha-g(\delta_{\mathrm{inject}},n)\le0$} would correspond to the formal infinite-interval correction.} The infinite-interval behavior seen on the real data (Table~\ref{tab:real}) is therefore a consequence of the small calibration set: at $n_{\mathrm{cal}}\approx40$ the nonzero contamination lowers the corrected level enough that too few scores remain for a finite quantile, and increasing $n_{\mathrm{cal}}$ at the same $\delta$ restores feasibility (Figure~\ref{fig:feasibility}).

\begin{table}[h]
\centering
\caption{Corrected procedure under extended injected contamination ($\delta$ up to $0.70$; $p=200$, $|\calA|=150$, $100$ seeds). {Finite-width rows use the clipped level $\alpha'\ge0.01$. Rows are contamination levels rather than competing methods, so widths are not ranked.}}
\label{tab:delta_extended}
\footnotesize
\begin{tabular}{lccc}
\toprule
$\delta_{\mathrm{inject}}$ & Coverage & Width & Feasibility \\
\midrule
$0.00$ & $\okcov{0.905}$ & $3.38$ & $100.0\%$ \\
$0.05$ & $\okcov{0.955}$ & $4.09$ & $100.0\%$ \\
$0.10$ & $\okcov{0.990}$ & $5.52$ & $99.9\%$ \\
$0.15$ & $\okcov{0.989}$ & $5.48$ & $99.9\%$ \\
$0.20$ & $\okcov{0.989}$ & $5.44$ & $99.9\%$ \\
$0.30$ & $\okcov{0.988}$ & $5.35$ & $99.9\%$ \\
$0.40$ & $\okcov{0.986}$ & $5.25$ & $99.9\%$ \\
$0.50$ & $\okcov{0.984}$ & $5.12$ & $99.9\%$ \\
$0.60$ & $\okcov{0.979}$ & $4.96$ & $99.9\%$ \\
$0.70$ & $\okcov{0.971}$ & $4.75$ & $99.9\%$ \\
\bottomrule
\end{tabular}
\end{table}

\textbf{Nonlinear mechanisms and residual scores.}
We replace the linear SEM $V=B^\top V+\varepsilon$ with additive nonlinear mechanisms $V_v=\sum_{u\in\mathrm{pa}(v)} B_{uv}\,\tanh(V_u)+\varepsilon_v$, and evaluate under two nonconformity scores: the synthetic score used in the main experiments and a genuine ridge-regression residual score $R_i^{(a)}=|V_i-\hat f_i(V_{-i})|$ with $\hat f_i$ fit on observational data ($p=200$, $|\calA|=150$, $20$ seeds). Table~\ref{tab:nonlinear} shows that {Corrected remains the most conservative method while Oracle, Estimated, and Pooled stay close to nominal coverage}, confirming that the method's behavior does not depend on the linear-Gaussian form of the synthetic generator.

\begin{table}[h]
\centering
\caption{Nonlinear SEM with additive $\tanh$ mechanisms, under the synthetic score and a real ridge-residual score ($p=200$, $|\calA|=150$, $20$ seeds). Coverage / width per method. {Ties at displayed precision are marked together.}}
\label{tab:nonlinear}
\footnotesize
\begin{tabular}{lcccc}
\toprule
& \multicolumn{2}{c}{Synthetic score} & \multicolumn{2}{c}{Ridge-residual score} \\
\cmidrule(lr){2-3}\cmidrule(lr){4-5}
Method & Coverage & Width & Coverage & Width \\
\midrule
Oracle & $\okcov{0.903}$ & $\secondw{3.34}$ & $\okcov{0.907}$ & $\bestw{4.35}$ \\
Estimated & $\okcov{0.901}$ & $\bestw{3.32}$ & $\okcov{0.906}$ & $\bestw{4.35}$ \\
Pooled & $\okcov{0.901}$ & $\bestw{3.32}$ & $\okcov{0.906}$ & $\bestw{4.35}$ \\
Corrected & $\okcov{0.920}$ & $3.59$ & $\okcov{0.924}$ & $4.39$ \\
\bottomrule
\end{tabular}
\end{table}
}

\section{Additional Discussion}\label{app:additional_discussion}
\subsection{{Future Work}}\label{app:future_work}

{\textbf{Calibration for descendant targets.}
The main theory focuses on test pairs with $Z_{a^\star,i}=0$, where interventions leaving target $i$ unchanged define a natural exchangeable stratum. Developing valid and efficient strata for descendant targets, where $Z_{a^\star,i}=1$, remains open.}

{\textbf{Weighted selective calibration.}
Algorithm~\ref{alg:local_icp} gives a distance estimator that can define kernel weights over interventions, but a finite-sample weighted-conformal guarantee with estimated causal distances remains to be proved.}

{\textbf{Deployable contamination bounds.}
Our correction uses an upper bound on the selected-set contamination fraction $\delta$; in the experiments this bound is supplied by true synthetic labels or real-data proxy labels. A fully deployable version would need data-driven high-confidence bounds on $\delta$, for example by combining descendant-learner uncertainty, held-out validation perturbations, or external biological annotations.}

{\textbf{Context-aware Mondrian strata.}
Real perturbation screens may violate exchangeability because of cell type, batch, tissue, dose, or donor effects. A natural extension is batch- or cell-type-stratified Mondrian calibration, potentially using multi-context resources such as scPerturb~\citep{peidli2024scperturb}.}

{\textbf{Combinatorial perturbations and active design.}
Our notation treats $a$ as a single intervention condition. Extending descendant learning and contamination control to combinatorial perturbations, and combining the method with active experimental design for causal discovery~\citep{squires2020active,eberhardt2005number}, are important next steps.}

{\textbf{Larger real-data validation.}
The real-data experiment uses proxy descendant labels and a small calibration set. Larger perturbation atlases with biologically validated causal annotations would allow a sharper test of both coverage and interval efficiency.}

\end{document}